\documentclass[final,3p,times]{elsarticle}
\usepackage{amssymb}
\usepackage{amsmath}
\usepackage{amsthm}
\usepackage{booktabs}
\usepackage{cleveref}
\usepackage{xcolor}
\journal{Computer Methods in Applied Mechanics and Engineering}
\usepackage{subfig}
\usepackage[numbers]{natbib}
\usepackage{graphicx}
\usepackage{url}
\usepackage{algorithmic}
\usepackage{algorithm}
\usepackage{float}

\newcommand{\specialcomment}[1]{{\textcolor{blue}{#1}}}
\DeclareMathOperator*{\argmin}{arg\,min} % thin space, limits underneath in displays

\newtheorem{theorem}{Theorem}

\begin{document}
\begin{frontmatter}

\title{Basis-to-Basis Operator Learning Using Function Encoders} 

\author[1]{Tyler Ingebrand\fnref{fn1}}
\ead{tyleringebrand@utexas.edu}

\author[1]{Adam J. Thorpe\corref{cor1}\fnref{fn1}} 
\ead{adam.thorpe@austin.utexas.edu}

\author[2]{Somdatta Goswami}
\ead{sgoswam4@jhu.edu}

\author[1]{Krishna Kumar}
\ead{krishnak@utexas.edu}

\author[1]{Ufuk Topcu}
\ead{utopcu@utexas.edu}

\cortext[cor1]{Corresponding author.}
\fntext[fn1]{These authors contributed equally to this work.}

%% Author affiliation
\affiliation[1]{organization={Oden Institute for Computational Engineering \& Science},
            addressline={University of Texas at Austin}, 
            city={Austin},
            postcode={78712}, 
            state={TX},
            country={USA}}
\affiliation[2]{organization={Civil and Systems Engineering Department},
            addressline={Johns Hopkins University}, 
            city={Baltimore},
            postcode={21218}, 
            state={MD},
            country={USA}}

%% Abstract
\begin{abstract}
We present Basis-to-Basis (B2B) operator learning, a novel approach for learning operators on Hilbert spaces of functions based on the foundational ideas of function encoders. We decompose the task of learning operators into two parts: learning sets of basis functions for both the input and output spaces and learning a potentially nonlinear mapping between the coefficients of the basis functions. B2B operator learning circumvents many challenges of prior works, such as requiring data to be at fixed locations, by leveraging classic techniques such as least squares to compute the coefficients. It is especially potent for linear operators, where we compute a mapping between bases as a single matrix transformation with a closed-form solution. Furthermore, with minimal modifications and using the deep theoretical connections between function encoders and functional analysis, we derive operator learning algorithms that are directly analogous to eigen-decomposition and singular value decomposition. We empirically validate B2B operator learning on seven benchmark operator learning tasks and show that it demonstrates a two-orders-of-magnitude improvement in accuracy over existing approaches on several benchmark tasks. 
\end{abstract}

%%Graphical abstract
% \begin{graphicalabstract}
% % \includegraphics{grabs}
% \includegraphics[]{images/fe_operator_title_figure.pdf}
% \end{graphicalabstract}

% %Research highlights
% \begin{highlights}
% \item Research highlight 1
% \item Research highlight 2
% \end{highlights}

%% Keywords

\end{frontmatter}

%%%%%%%%%%%%%%%%%%%%%%%%%
%%%%%%%%%%%%%%%%%%%%%%%%%

\section{Introduction}
\label{sec:intro}

Operator learning has emerged as a significant paradigm in scientific machine learning, offering tools to approximate mappings between infinite-dimensional function spaces. It is particularly effective in solving partial differential equations (PDEs), where traditional numerical methods can be computationally intensive and may not scale well with increasing problem complexity. By learning operators that map input functions---such as initial and boundary conditions, source terms, or material properties---to output solutions, we can develop efficient surrogate models that offer rapid and accurate predictions for complex physical systems governed by PDEs.

Operator learning fundamentally involves mapping between infinite-dimensional function spaces, yet practical implementations require finite-dimensional representations to achieve computational tractability. Current approaches to operator learning broadly fall into two categories. The first category includes neural network-based methods like DeepONet \cite{lu2021learning}, which leverages the universal approximation theorem for operators \cite{chen1995universal}, and Fourier neural operators (FNO) \cite{li2021fourier}, which employ Fourier convolutions to compute integral transforms efficiently. Neural operator-based approaches rely on large neural architectures and GPU acceleration to approximate high-dimensional function transformations. 
Notably, neural network-based approaches do not explicitly leverage the geometric structure inherent in the underlying Hilbert spaces when learning the mapping between infinite and finite-dimensional representations. 
Hilbert spaces are equipped with a norm and a well-defined inner product that enable us to measure the norm or magnitude of the functions and determine the similarity or angle between functions, yielding richer geometric structures that should be exploited.
The second category of operator learning techniques encompasses reduced order model (ROM)-based methods \cite{geelen2023operator,kontolati2023influence}, which provide a more mathematically principled approach to dimensional reduction through techniques like proper orthogonal decomposition (POD) and principal component analysis (PCA). While ROMs offer rigorous mathematical foundations for representing finite-dimensional approximations of function spaces, they typically require fixed discretization schemes. 
This suggests an opportunity to develop new methods that exploit the geometric properties of the underlying Hilbert space to inform and constrain the learned representations.

We propose an algorithm that combines the mathematical principles of basis function learning with neural networks' computational flexibility and scalability. Our method eliminates the need for fixed meshes in input and output function spaces, offering greater versatility in practical applications. Specifically, we focus on learning nonlinear operators on Hilbert spaces using the theory of function encoders \cite{ingebrand2024zeroshottransferneuralodes, ingebrand2024zeroshotreinforcementlearningfunction} with neural network basis functions. By leveraging the inner product structure inherent in Hilbert spaces, our approach provides an intuitive notion of function projections. The Hilbert space framework unlocks powerful spectral theory tools, including eigen-decompositions and singular value decompositions, which we use to analyze and interpret learned operators. Furthermore, the geometric properties of Hilbert spaces enhances optimization by providing a natural notion of angle and distance, leading to better-behaved gradients and more stable training dynamics. These properties allow us to develop operator learning models that are both interpretable and analytically tractable while maintaining the flexibility to capture complex, nonlinear transformations.

\begin{figure}
    \centering
    \includegraphics[]{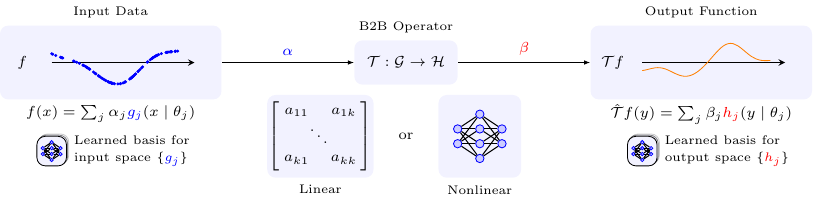}
    \caption{The basis-to-basis (B2B) operator maps the coefficients $\alpha$ of the input function $f$ to the coefficients $\beta$ of the output function $\mathcal{T}f$.}
    \label{fig:schematic}
\end{figure}

The function encoder-based operator learning approach consists of two main stages. First, it learns a set of basis functions to represent the input and output function spaces efficiently. This approach leverages function encoders---neural network-parameterized basis functions---to span the relevant function spaces. Then, it learns an operator that maps between these function representations.  We propose three variations:
\begin{itemize}
    \item \textbf{Basis-to-Basis (B2B)} - 
    The standard B2B operator learning approach involves two main steps:
    (i) Basis training: Independently learns a basis for each function space; (ii) Mapping: Establishes a mapping between these bases. For linear operators, the mapping is a matrix transformation, while for nonlinear operators, we use a deep neural network. This approach offers flexibility in handling linear and nonlinear operators while maintaining a clear separation between the function space representation and the operator mapping. 
    
    \item \textbf{Singular Value Decomposition (SVD)} - This variant of our approach incorporates the mathematical structure of SVD into the learning process, enabling the model to learn an SVD representation of the operator in an end-to-end manner. One of its key strengths lies in its ability to accommodate different input and output function spaces, making it versatile for a range of applications where input and output domains may differ. However, it is important to note that this approach is restricted to linear operators, as the SVD is inherently a linear decomposition.
    
    \item \textbf{Eigen-Decomposition (ED)} - This variant focuses on learning an eigen-decomposition of the operator in an end-to-end manner. This variant has specific constraints: it requires the input and the output function to be defined on the same bounded domain, limiting its application to scenarios where the input and output domains are the same. Additionally, like the SVD variant, it is restricted to linear operators due to the inherent nature of eigen-decomposition.
\end{itemize}
Despite their restriction to linear operators, SVD and ED provide powerful and interpretable representations of linear transformations between function spaces. This interpretability is particularly significant in cases when the eigenvectors and eigenvalues of the operator have physical or mathematical significance. We demonstrate the performance of B2B, SVD, and ED through numerous benchmark examples from the literature. Furthermore, we have also conducted linear operator analysis to support our findings. 

% Highlights

The paper is organized as follows. In Section \ref{sec:related_work}, we have provided a taxonomy of representative works in the literature. Section \ref{sec:FE_proposed} introduces our proposed approaches, basis-to-basis operator learning (B2B) and its variants. In Section \ref{sec:results}, we compare the performance of B2B, SVD, and ED with an existing popular neural operator, DeepONet, for seven benchmark problems from the literature. In Section \ref{sec:ablation}, we evaluate our approach's sensitivity to hyper-parameters. Finally, we summarize our observations and provide concluding remarks in Section \ref{sec:conclusion}. The code and the associated datasets are available at \url{https://github.com/tyler-ingebrand/OperatorFunctionEncoder}.

\section{Related Work}
\label{sec:related_work}

DeepONet and its variants \cite{transformerOperator, HE2024107258, KUSHWAHA2024104266, lu2021learning, recurrent, riemann, srno,     ZHANG2024112638,  d2no} employ dual-network architectures, combining a branch and trunk network through a dot product to approximate nonlinear mappings between function spaces. 
These approaches are highly general due to a nonlinear coefficient mapping that takes input function data at fixed measurement locations (``sensors"). 
Additional architectural advancements that apply to various application domains or propose specialized architectures have been suggested. 
A parallel development in neural operators has emerged through spectral methods:  FNO\cite{li2021fourier}, wavelet neural operators \cite{tripura2023wavelet}, and Laplace neural operators \cite{cao2024laplace} that leverage integral transforms to capture global dependencies in the solution space. 
However, these spectral approaches face significant challenges with the curse of dimensionality, requiring exponentially larger datasets to accurately capture underlying patterns. Moreover, current neural operator approaches often underutilize both the inherent operator structure and the finite-dimensional representation space.
Recent work focuses on enhancing these approaches to improve their accuracy and generalizability by incorporating physics knowledge into the network architecture and training procedure \cite{pideepo, pideeponetstudy, KORIC2023123809, li2023phasefielddeeponetphysicsinformeddeep}. In principle, these physics-informed constraints can be naturally integrated into function encoders, i.e., learned basis functions, using the same strategies. 

Reduced-order methods (ROMs) provide an elegant framework for approximating nonlinear operators by projecting functional data onto low-dimensional subspaces. These methods rely on mesh-based discretization and decomposition techniques, such as proper orthogonal decomposition or principal component analysis, to extract the dominant modes or basis functions. 
Working within these finite-dimensional representations, ROMs significantly reduce the computational complexity of the operator learning problem. However, a key limitation of ROM-based approaches that they share with DeepONet and FNO is that they typically depend on fixed, discrete grids or meshes. 
% While this may be fine in some settings, there are some settings where this is not fine.
In contrast, our approach does not depend on fixed sensor locations since we compute transformations on learned representations of the input function.

Despite their generality, existing methods rely heavily on brute-force approximations and lack explicit mathematical structures, particularly Hilbert space inner products. This absence limits the use of fundamental geometric concepts such as orthogonality and projections, potentially compromising model scalability, interpretability, and stability \cite{reno}. While finite-dimensional reduced representations \cite{rompca, BennerGW15, SWISCHUK2019704} address some of these challenges by projecting functions onto finite sets of basis functions—thereby preserving essential features while reducing computational complexity \cite{ gappypod, reducedbasis, podnn, podapplication}—these approaches typically remain constrained by fixed meshes in both input and output spaces. Our approach overcomes these limitations by explicitly imposing structure on the learned representations through a well-defined inner product, enabling basis learning and coefficient mapping through least-squares optimization.

Recent advances in operator learning have explored the integration of neural networks with principled dimensionality reduction approaches. Several notable approaches have emerged in this direction. NOMAD~\cite{nomad} explores mappings between function spaces by encoding input functions into finite-dimensional embeddings and employing a nonlinear decoder that directly maps these embeddings to output functions, enabling efficient representation of functions that lie on nonlinear submanifolds. Similarly, CORAL~\cite{coral} employs an autoencoder-like approach to learn functional representations. Other researchers have combined PCA-based dimensionality reduction on fixed meshes with neural networks to approximate mappings in Hilbert spaces \cite{pca_nn}. 
Building on these foundations, some methods have enhanced the DeepONet architecture by incorporating geometric information as additional input \cite{He_2024, dimon}. 
Independently from our work, a similar approach was proposed that overcomes the need for fixed sensor locations using dictionary learning to constructively compute a low-dimensional projection onto a reduced set of basis functions \cite{bahmani2024resolution}. This approach uses principles from dictionary learning to learn orthogonal sinusoidal representation networks as an implicit neural representation of dictionary atoms via an iterative training procedure.

In contrast to prior work, our approach explicitly imposes structure on the operator learning solution. We leverage the principles of Hilbert spaces in learning the finite-dimensional representation and address the key limitation of fixed meshes by learning sets of neural network basis functions via simple gradient descent and least-squares optimization.
While many existing approaches employ Hilbert spaces, our explicit utilization of the Hilbert space structure to learn the representations provides a highly interpretable solution and enhances both the performance and stability of the learned mapping.

\section{Learning Nonlinear Operators Using Function Encoders}
\label{sec:FE_proposed}

Let the input function space $\mathcal{G} = \lbrace f \mid f : \mathcal{X} \to \mathcal{Z} \rbrace$ and the output function space $\mathcal{H} = \lbrace f \mid f : \mathcal{Y} \to \mathcal{W} \rbrace$ be Hilbert spaces, where $\mathcal{W}, \mathcal{X}, \mathcal{Y},$ and $\mathcal{Z}$ are Euclidean. We consider the problem of numerically estimating a continuous operator $\mathcal{T}: \mathcal{G} \to \mathcal{H}$. 
Formally, the operator learning problem is: given pairs $(f, h) \in \mathcal{G} \times \mathcal{H}$, and a potentially nonlinear operator $\mathcal{T}: \mathcal{G} \to \mathcal{H}$ such that $\mathcal{T}f = h$, the objective is to find an estimate $\hat{\mathcal{T}}$ of $\mathcal{T}$ such that $\hat{\mathcal{T}}f \approx \mathcal{T}f$ for all $f \in \mathcal{G}$.

We presume that during training, we have access to a set of datasets $\mathcal{D} = \lbrace D_{1}, D_{2}, \ldots, D_{N}\rbrace$. Each dataset $D_{n}$, $n = 1, \ldots, N$, consists of input-output pairs for both input and output function spaces, $D_n = \big ( \big \{ \big (x_i, f_n(x_i) \big ) \big \}_{i=1}^m, \big \{ \big (y_i, \mathcal{T}f_n(y_i)  \big ) \big \}_{i=1}^p \big )$.
The operator data $\mathcal{T}f_{n} \in \mathcal{H}$ can be obtained, for instance, using classical numerical methods. Note that the example points $x_{i}$ do not need to be the same across all datasets $D_{n}$. Indeed, a key advantage of our approach relative to prior works is that these sample points are not fixed. 
After training, we only have access to sample data consisting of input-output pairs $\lbrace \big (x_{i}, f(x_{i}) \big )\rbrace_{i=1}^{m}$ taken from a new, unseen function $f \in \mathcal{G}$ that we want to evaluate. Given this data, the goal is to infer $(\mathcal{T}f)(y)$ at any given evaluation point $y \in \mathcal{Y}$.

Our approach can be interpreted in two distinct parts: learning basis functions for both the input and output domains via function encoders \cite{ingebrand2024zeroshotreinforcementlearningfunction}, and then learning a mapping from the function representations in the input domain to the function representations in the output domain. We call our approach basis-to-basis (B2B) operator learning since we effectively learn an operator to map between two learned bases.

\subsection{Learning a Basis for Input and Output Domains}

To learn the input and the output domains, we use the principles of function encoders \cite{ingebrand2024zeroshotreinforcementlearningfunction}. This approach enables us to represent a function via a learned set of basis functions.
We first consider the input domain $\mathcal{G}$.
Function encoders learn a set of $k$ basis functions $g_{1}, \ldots, g_{k}$ that are parameterized by neural networks to span a Hilbert space $\mathcal{G}$. 
The functions $f \in \mathcal{G}$ can be represented as a linear combination of basis functions,
\begin{equation}
    \label{eqn: linear combination source}
    f(x) = \sum_{j=1}^{k} \alpha_{j} g_{j}(x \mid \theta_{j}),
\end{equation}
where $\alpha \in \mathbb{R}^{k}$ are real-valued coefficients and $\theta_{j}$ are the network parameters for $g_{j}$.

The coefficients $\alpha \in \mathbb{R}^{k}$ can be computed as the solution to a least-squares problem,
\begin{equation}
    \label{eqn: least-squares estimate}
    \min_{\alpha} \frac{1}{m} \sum_{i=1}^{m} \biggl \lVert f(x_{i}) - \sum_{j=1}^{k} \alpha_{j} g_{j}(x_{i} \mid \theta_{j}) \biggr\rVert_2^{2}.
\end{equation}
Let $G \in \mathbb{R}^{m \times k}$ be a matrix with elements $G_{ij} = g_{j}(x_{i} \mid \theta_{j})$. The solution to \eqref{eqn: least-squares estimate} can be computed in closed-form as $\alpha = (G^{\top} G)^{-1} G^{\top} \boldsymbol{f}$, where $\boldsymbol{f}$ is a vector with elements $\boldsymbol{f}_{i} = f(x_{i})$. Note that the matrix inversion $(G^{\top} G)^{-1}$ is $k \times k$, meaning it can be computed efficiently for a small number of basis functions, e.g. on the order of $100$. 

We presume access to a dataset $\mathcal{D}$ to train the function encoder. 
For each function $f_{n}$ and corresponding dataset $D_{n}$ in $\mathcal{D}$, the coefficients $\alpha$ are computed using \eqref{eqn: least-squares estimate}, and then we obtain an empirical estimate of $f$ via \eqref{eqn: linear combination source}. The error is computed using the Euclidean norm. The loss is simply the mean of the errors for all functions $f_{n}$ in $\mathcal{D}$, and is minimized via gradient descent \cite{ingebrand2024zeroshotreinforcementlearningfunction}, i.e.
\begin{equation}
    \label{eqn: function encoder loss}
    L = \frac{1}{Nm}\sum_{n=1}^N \sum_{i=1}^{m} \biggl \lVert f_{n}(x_{i}) - \sum_{j=1}^{k} \alpha_{j} g_{j}(x_i \mid \theta_{j}) \biggr \rVert_2^{2}.
\end{equation}
During training, the basis functions learn to align with the subspace spanned by the data, and
the least-squares method in \eqref{eqn: least-squares estimate} only requires the basis functions to be linearly independent \cite{ingebrand2024zeroshottransferneuralodes}. 
The training procedure is depicted in Figure \ref{fig: function encoder training}.
After training, the basis function networks are fixed. Then, to predict the outputs for a new function $f_{new} \in \mathrm{span}\lbrace g_{1}, \ldots, g_{k} \rbrace$, we can compute the coefficients using \eqref{eqn: least-squares estimate} with a dataset $D_{new}$, and compute the predictions as the linear combination of fixed basis functions (Figure \ref{fig: function encoder inference}). 
This is an important point since the coefficients are computed without retraining or fine-tuning. Furthermore, these coefficients are a fully informative representation of the function since they can reproduce the function for any new data point.

\begin{figure}
    \centering
    \includegraphics[]{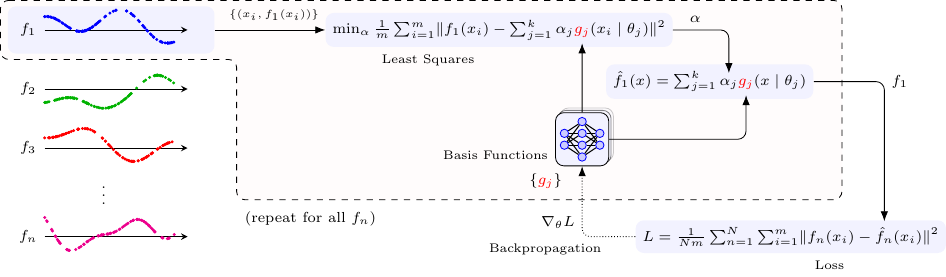}
    \caption{\textit{The function encoder training procedure}. We compute estimates $\hat{f}_{1}, \ldots, \hat{f}_{n}$ of multiple functions $f_{1}, \ldots, f_{n} \in \mathcal{F}$ using data. Then, we compute the total loss of the overall functions and backpropagate the gradient with respect to the basis function parameters $\theta$ using gradient descent.}
    \label{fig: function encoder training}
\end{figure}

\begin{figure}
    \centering
    \includegraphics[]{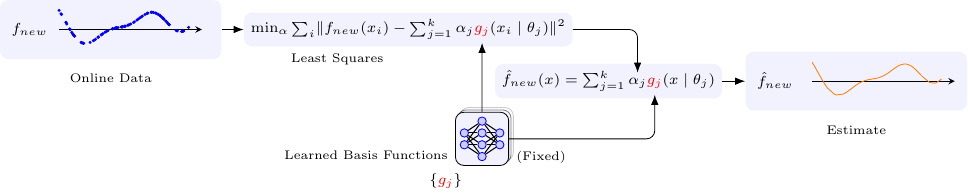}
    \caption{The function encoder inference procedure. Using online data from a new, unseen function $f_{new}$, we compute the corresponding coefficients $\alpha \in \mathbb{R}^{k}$ using least squares. The estimate $\hat{f}_{new}$ of $f_{new}$ is a linear combination of the fixed basis functions $\lbrace g_{j} \rbrace$ with the coefficients $\alpha$.}
    \label{fig: function encoder inference}
\end{figure}

We use an identical training procedure to learn a set of basis functions $h_{1}, \ldots, h_{\ell}$ to span the codomain of the operator $\mathcal{T}$. 
The transformed function $\mathcal{T}f \in \mathcal{H}$ can be represented via a linear combination of basis functions $h_{1}, \ldots, h_{\ell}$, 
\begin{equation}
    \label{eqn: linear combination target}
    (\mathcal{T}f)(y) = \sum_{j=1}^{\ell} \beta_{j} h_{j}(y \mid \vartheta_{j}),
\end{equation}
and we compute the corresponding coefficients $\beta \in \mathbb{R}^{\ell}$ for $\mathcal{T}f$ as before using \eqref{eqn: least-squares estimate}. 
We obtain an empirical estimate of the coefficients using data pairs $\lbrace \big (y_{i}, (\mathcal{T}f)(y_{i}) \big ) \rbrace_{i=1}^{p}$. 
We train the basis functions using the dataset $\mathcal{D}$, where for each function $f_{n}$ and corresponding dataset $D_{n} \in \mathcal{D}$, we use the evaluations of the operator action $(\mathcal{T}f_{n})(y_{i})$ to compute the error, before accumulating the loss and updating via gradient descent as in \eqref{eqn: function encoder loss}. 
After training, we obtain a means to represent the operator action $\mathcal{T}f \in \mathrm{span}\lbrace h_{1}, \ldots, h_{\ell} \rbrace$ of a new function $f$ via the coefficient representation $\beta$.

This is key to our approach since we can now view the infinite-dimensional operator $\mathcal{T}: \mathcal{G} \to \mathcal{H}$ as a finite-dimensional mapping of coefficients $\alpha \mapsto \beta$.

\subsection{Nonlinear Operators Between Learned Spaces}

Once the basis functions $g_{1}, \ldots, g_{k}$ and $h_{1}, \ldots, h_{\ell}$ are trained, our goal is to learn a nonlinear mapping from the coefficient representation $\alpha \in \mathbb{R}^{k}$ of a function $f \in \mathcal{G}$ to a corresponding representation $\beta \in \mathbb{R}^{\ell}$ of the transformed function $\mathcal{T}f \in \mathcal{H}$.
To model the nonlinear operator $\mathcal{T}$, we use a simple feed-forward neural network to model the relation $\hat{\mathcal{T}}: \alpha \mapsto \beta$, though we are free to choose any architecture or nonlinear regression technique.

We can use the dataset $\mathcal{D}$ to train the operator network.
For each function $f_{n}$ and the corresponding dataset $D_{n}$, we compute the coefficient representation $\alpha$ using example data $\lbrace \big (x_{i}, f_{n}(x_{i}) \big) \rbrace_{i=1}^{m}$ and the corresponding $\beta$ using  data $\lbrace \big (y_{i}, (\mathcal{T}f_{n})(y_{i}) \big )\rbrace_{i=1}^{p}$ using \eqref{eqn: least-squares estimate}. We then perform a forward pass of the network to obtain the coefficients estimate $\hat{\beta}=\hat{\mathcal{T}}(\alpha)$. 
The loss is the error of our coefficients estimate $\lVert \beta - \hat{\beta} \rVert^{2}_2$, which is minimized via gradient descent. In other words, we perform standard regression from $\alpha$ to $\beta$. Please see Algorithm \ref{alg:b2b} for the full B2B algorithm.

\begin{algorithm}[!h]
  \caption{Basis-to-Basis Operator Learning}
  \fontsize{10.0pt}{10.0pt}\selectfont
  \label{alg:b2b}
  \begin{algorithmic}[1]

    \STATE \textbf{Input:} Step size $\eta$, set of data sets $\mathcal{D}$
    \STATE Initialize $\{g_j\}$ parameterized by $\{\theta_j\}$ \specialcomment{\# Input function space basis}
    \STATE Initialize $\{h_j\}$ parameterized by $\{\vartheta_{j}\}$  \specialcomment{\# Output function space basis}
    \STATE Initialize $\hat{\mathcal{T}}$ parameterized by $\phi$  \specialcomment{\# Operator Network}

    \specialcomment{\# Train the input function space basis via the function encoder algorithm.}
    \WHILE{not converged}
        \STATE Loss $L = 0$
        \FOR{$\{\big (x_i, f_n(x_i) \big ) \}_{i=1}^m $ in $\mathcal{D}$}
            \STATE $\alpha = \argmin_{\alpha} \frac{1}{m} \sum_{i=1}^{m} \biggl \lVert f_n(x_{i}) - \sum_{j=1}^{k} \alpha_{j} g_{j}(x_{i} \mid \theta_{j}) \biggr\rVert_2^{2}$
            \STATE $\hat{f}_n = \sum_j \alpha_j g_j$
            \STATE $L \mathrel{+}= \frac{1}{Nm}\sum_{i=1}^m \biggl\lVert f_n(x_i) - \hat{f}_n(x_i) \biggr\rVert_2^2$ 
        \ENDFOR
        \STATE $\theta_j = \theta_j - \eta \nabla_{\theta_j} L$
    \ENDWHILE

    \specialcomment{\# Train the output function space basis via the function encoder algorithm.}
    \WHILE{not converged}
        \STATE Loss $L = 0$
        \FOR{$\{\big (y_i, \mathcal{T}f_n(y_i) \big ) \}_{i=1}^p $ in $\mathcal{D}$}
            \STATE $\beta = \argmin_{\beta} \frac{1}{p} \sum_{i=1}^{p} \biggl \lVert \mathcal{T}f_n(y_{i}) - \sum_{j=1}^{\ell} \beta_{j} h_{j}(y_{i} \mid \vartheta_{j}) \biggr\rVert_2^{2}$
            \STATE $\widetilde{\mathcal{T}f_n} = \sum_j \beta_j h_j$
            \STATE $L \mathrel{+}= \frac{1}{Np}\sum_{i=1}^p \biggl\lVert \mathcal{T}f_n(y_i) - \widetilde{\mathcal{T}f_n}(y_i) \biggr\rVert_2^2$ 
        \ENDFOR
        \STATE $\vartheta_j = \vartheta_j - \eta \nabla_{\vartheta_j} L$
    \ENDWHILE

    \specialcomment{\# Train the operator network via standard gradient descent.}
    \WHILE{not converged}
        \STATE Loss $L = 0$
        \FOR{$\{\big (x_i, f_n(x_i) \big )\}_{i=1}^m, \{\big (y_i, \mathcal{T}f_n(y_i) \big )\}_{i=1}^p $ in $\mathcal{D}$}
            \STATE $\alpha = \argmin_{\alpha} \frac{1}{m} \sum_{i=1}^{m} \biggl \lVert f_n(x_{i}) - \sum_{j=1}^{k} \alpha_{j} g_{j}(x_{i} \mid \theta_{j}) \biggr\rVert_2^{2}$
            \STATE $\beta = \argmin_{\beta} \frac{1}{p} \sum_{i=1}^{p} \biggl \lVert \mathcal{T}f_n(y_{i}) - \sum_{j=1}^{\ell} \beta_{j} h_{j}(y_{i} \mid \vartheta_{j}) \biggr\rVert_2^{2}$
            \STATE $\hat{\beta} = \hat{\mathcal{T}}(\alpha | \phi)$
            \STATE $L \mathrel{+}= \frac{1}{N} \biggl\lVert \beta - \hat{\beta} \biggr\rVert_2^2$ 
        \ENDFOR
        \STATE $\phi = \phi - \eta \nabla_{\phi} L$
    \ENDWHILE
    \STATE \textbf{return} $\lbrace g_j \rbrace, \lbrace h_j \rbrace, \lbrace \hat{\mathcal{T}} \rbrace$ 
  \end{algorithmic}
\end{algorithm}

\subsection{Learning Linear Operators Via Least Squares}
\label{section: linear operators}

Linear operators are a special case of our proposed approach since the transformation of coefficients from the input to the output domain can be viewed as a simple matrix multiplication. 
Given a trained basis $g_{1}, \ldots, g_{k}$ and $h_{1}, \ldots, h_{\ell}$, we learn a linear mapping $\alpha \mapsto \beta$ as a matrix $A \in \mathbb{R}^{\ell \times k}$. 
Using the dataset $\mathcal{D}$ and the trained basis functions for the input and output domains, we compute the coefficient representations $\alpha_{n}$ and $\beta_{n}$ for $f_{n}$ and $\mathcal{T}f_{n}$, where the subscript $n$ here denotes the sample index, not the $n^{\rm th}$ element.
Then, using data $\lbrace (\alpha_{n}, \beta_{n})_{n=1}^N \rbrace$, we compute the matrix $A$ as the solution to the following least-squares optimization problem,
\begin{equation}
    \label{eqn: linear operator least squares}
    \min_{A} \frac{1}{N} \sum_{n=1}^N \lVert \beta_{n} - A \alpha_{n} \rVert^{2}.
\end{equation}

The linear operator gives us generalization performance across the input and output space. In fact, due to the linear properties of the least-squares solution and the matrix transformation, we arrive at the following lemma:

\begin{theorem} \label{theorem}
  If $f_3 := a f_1 + b f_2$, $a,b\in \mathbb{R}$, $f_1, f_2 \in \mathcal{G}$, and $\mathcal{T}$ is a linear operator, then $\hat{\mathcal{T}}f_3 = a \hat{\mathcal{T}}f_1 + b \hat{\mathcal{T}}f_2$. 
\end{theorem}

\begin{proof} 
    Let $f_{1}, f_{2}, f \in \mathcal{G}$, $c \in \mathbb{R}$. 
    Let $\hat{\mathcal{T}} = A \in \mathbb{R}^{\ell \times k}$ as in \eqref{eqn: linear operator least squares}.
    
    \textit{Additivity:} 
    From \eqref{eqn: least-squares estimate}, we have that for any $f \in \mathcal{G}$, the corresponding coefficients $\alpha = (G^{\top} G)^{-1} G^{\top} \boldsymbol{f}$.  
    Since $\hat{\mathcal{T}} f = H^{\top} \beta$ from \eqref{eqn: linear combination target}, where $H = [h_{1}, \ldots, h_{\ell}]^{\top}$ and $\beta = A \alpha$, then $\hat{\mathcal{T}} f = H^{\top} (A \alpha)$, and we have that
    \begin{equation}
        \hat{\mathcal{T}}(f_{1} + f_{2}) = H^{\top} A (G^{\top} G)^{-1} G^{\top} (\boldsymbol{f}_{1} + \boldsymbol{f}_{2}) = H^{\top} A (G^{\top} G)^{-1} G^{\top} \boldsymbol{f}_{1} + H^{\top} A (G^{\top} G)^{-1} G^{\top} \boldsymbol{f}_{2} = \hat{\mathcal{T}}f_{1} + \hat{\mathcal{T}}f_{2}.
    \end{equation}

    \textit{Homogeneity:} 
    For any $f \in \mathcal{G}$, we have that $cf = c (G^{\top} G)^{-1} G^{\top} \boldsymbol{f} = (G^{\top} G)^{-1} G^{\top} (c \boldsymbol{f})$. Thus, from \eqref{eqn: linear combination target}, we have that
    \begin{equation}
        \hat{\mathcal{T}}(cf) = H^{\top} A (G^{\top} G)^{-1} G^{\top} (c \boldsymbol{f}) = c H^{\top} A (G^{\top} G)^{-1} G^{\top} \boldsymbol{f} = c \hat{\mathcal{T}}f.
    \end{equation}

    This concludes the proof.
\end{proof}

This theorem implies that we have guaranteed generalization within the function spaces for linear operators. In practice, there is error due to sampling when the $x_{i}$'s are sampled at different locations between datasets.

\subsection{Singular Value Decomposition and Eigendecompositions}

We now consider two alternative methods for computing linear operators using function encoders that provide greater analytical information, at the expense of prediction accuracy. We focus on operator analysis using singular value decomposition and eigendecomposition, which provide information about the learned operator through the singular values or eigenvalues. 

\begin{figure}[!b]
    \centering
    \includegraphics[]{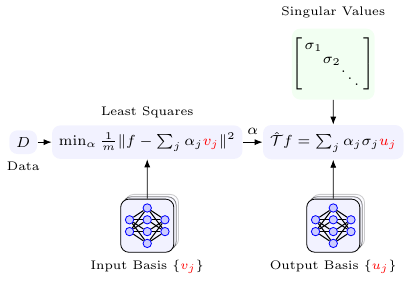} \hspace{0.5cm} \unskip\ \vrule\ \hspace{0.5cm}
    \includegraphics[]{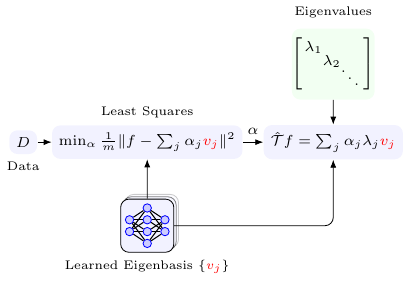}
    \caption{Left: The inference procedure for SVD. A dataset $D=\{\big(x_i, f(x_i)\big)\}_{i=1}^m$ is given which describes a function $f$. This data is used to compute the coefficients of the basis functions via the least-squares solution. Then, $\hat{\mathcal{T}}f$ is approximated using \eqref{eqn: svd operator}. 
    Right: The inference procedure for ED. The procedure is analogous, with the only change being that $\mathcal{T}$ is self-adjoint, and so the same basis is used for the input and output spaces. 
    }
    
    \label{fig:enter-label}
\end{figure}

\subsubsection{Using Function Encoders for Singular Value Decomposition (SVD)}

The singular value decomposition of an operator $\mathcal{T} = U \Sigma V^{\top}$ is a representation in terms of its singular values and singular vectors. For a compact, linear, bounded operator $\mathcal{T} : \mathcal{G} \to \mathcal{H}$, there exists a singular value decomposition such that
$\mathcal{T}f = \sum_{i} \sigma_i \langle f, v_i \rangle u_i$,
where $\sigma_{i} \in \mathbb{R}_{+}$ are positive real coefficients and $\lbrace u_{i} \rbrace$ and $\lbrace v_{i} \rbrace$ are orthonormal systems. The coefficients $\sigma$, called the singular values, and $\lbrace u_{i} \rbrace$ and $\lbrace v_{i} \rbrace$, called the left and right singular vectors, satisfy
\begin{equation}
    \label{eqn: operator on singular vectors}
    \mathcal{T} v_i = \sigma_i u_i.
\end{equation}
Since the right singular vectors $\lbrace v_i \rbrace$ form an orthonormal system, any function $f \in \mathcal{G}$ can be represented as a (scaled) linear combination,
\begin{equation}
    \label{eqn: svd function basis}
    f = \sum_{i} \langle f, v_i \rangle v_i.
\end{equation}
The coefficients of the function in the input domain are denoted by $\alpha_i = \langle f, v_i \rangle$.
Using \eqref{eqn: svd function basis} and \eqref{eqn: operator on singular vectors}, the operator action $\mathcal{T}f$ on a function can be decomposed as
\begin{equation}
    \label{eqn: svd operator}
    \mathcal{T}f = \sum_{i} \sigma_i \langle f, v_i \rangle u_i = \sum_{i} \sigma_i \alpha_i u_i.
\end{equation}

The learning procedure is as follows. First, we initialize the basis functions $\lbrace u_i \rbrace$ and $\lbrace v_i \rbrace$ as neural networks, as in Section \ref{sec:FE_proposed}. 
We also initialize $\lbrace \sigma_i \rbrace$ as learned parameters, i.e.\ $\sigma \in \mathbb{R}^k$. For any pair $(f, \mathcal{T}f)$ and the corresponding dataset $D = \big \{ \{x_i, f(x_i) \}_{i=1}^m, \{y_i, \mathcal{T}f(y_i)\}_{i=1}^p \big \}$, we compute 

\begin{equation}
    \label{eqn: least-squares estimate svd}
    \alpha = \argmin_{\alpha} \frac{1}{m} \sum_{i=1}^{m} \biggl \lVert f(x_{i}) - \sum_{j=1}^{k} \alpha_{j} v_{j}(x_{i} \mid \theta_{j}) \biggr\rVert_2^{2}.
\end{equation}

\noindent By using least squares to compute the coefficients rather than the inner product alone, we remove the requirement that the bases are orthonormal. Then, we use the right-hand side of \eqref{eqn: svd operator} to approximate the operator action, $\hat{\mathcal{T}}f = \sum_i \sigma_i \alpha_i u_i$. Lastly, our loss is simply $\lVert \mathcal{T}f - \hat{\mathcal{T}}f \rVert_2^2$ evaluated via $\{y_i, \mathcal{T}f(y_i)\}_{i=1}^p$. The singular vectors and values are trained to minimize this loss in an end-to-end fashion via gradient descent, analogous to the function encoder algorithm in \cite{ingebrand2024zeroshotreinforcementlearningfunction}. As a result, we simultaneously learn the left singular vectors, right singular vectors, and singular values for a linear operator acting on function spaces. We formalize this as Algorithm \ref{alg:svd}.

One drawback of SVD is that the training procedure does not necessarily learn to span the input domain well because we do not explicitly train the basis networks $v_{1}, \ldots, v_{k}$ to span the input domain $\mathcal{G}$. This end-to-end loss can lead to poorer generalization performance outside the training dataset.
However, this is offset by the fact that by explicitly learning the singular values, we enable further analysis of the operator, such as sensitivity analysis (via the condition number), stability analysis, or principal component analysis for model reduction.

\subsubsection{Using Function Encoders for Eigen-decompositions (ED)}

This approach also extends to learning eigendecompositions of compact, self-adjoint linear operators. The eigendecomposition of an operator $\mathcal{T}$ consists of a set of eigenvectors $\lbrace v_{i} \rbrace$ and a set of eigenvalues $\lambda_{i} \in \mathbb{R}$ such that $\mathcal{T}v_{i} = \lambda_{i} v_{i}$.
The operator action $\mathcal{T}f$ can be decomposed as
\begin{equation}
    \mathcal{T}(f) = \sum_{i} \lambda_{i} \alpha_{i} v_{i},
\end{equation}
where $\alpha$ are the coefficients of the function $f$ computed using \eqref{eqn: least-squares estimate}. 

Here, the input and output domain are the same, meaning we can use the same set of basis functions to represent the function $f$ and the operator action on the function $\mathcal{T}f$. 
Similar to before, we simultaneously learn a set of eigenvalues $\lambda_{i} \in \mathbb{R}$ during training using backpropagation, and the learned basis functions correspond to the eigenvectors of the operator. The algorithm for finding these basis functions is analogous to SVD with $u_i=v_i$. We formalize this in Algorithm \ref{alg:ed}.
The principal advantage of this approach is that it provides greater analytical insights for the learned operator.

The compactness of the learned derivative operator, as revealed by our SVD and ED approaches, offers significant insights into the nature of the differentiation process.  \Cref{fig:derivative_decay_comparison} illustrates the decay of eigenvalues and singular values for the derivative operator. The rapid decay observed in these plots, approximately following an exponential pattern ($\approx e^{-0.05x}$ for the matrix singular values), represents the compactness of the operator. This compactness suggests that a low-rank representation can efficiently approximate the derivative operator, which is crucial for computational efficiency in large-scale applications. Additionally, the smooth decay indicates that the operator has a strong smoothing effect, consistent with the well-known property that differentiation tends to amplify high-frequency components. The compactness also implies that the operator is bounded, which is important for ensuring stability in numerical implementations.

\begin{algorithm}[!h]
  \caption{SVD Extension of the Function Encoder Algorithm}
  \fontsize{10.0pt}{10.0pt}\selectfont
  \label{alg:svd}
  \begin{algorithmic}[1]

    \STATE \textbf{Input:} Step size $\eta$, set of data sets $\mathcal{D}$
    \STATE Initialize $\{v_j\}$ parameterized by $\{\theta_j\}$ 
    \STATE Initialize $\{\sigma_j\}$ 
    \STATE Initialize $\{u_j\}$ parameterized by $\{\vartheta_{j}\}$ 

    \WHILE{not converged}
        \STATE \specialcomment{\# Compute loss for all pairs $(f, \mathcal{T}f)$}
        \STATE Loss $L = 0$
        \FOR{$\{\big (x_i, f_n(x_i) \big )\}_{i=1}^m, \{\big (y_i, \mathcal{T}f_n(y_i) \big )\}_{i=1}^p $ in $\mathcal{D}$}
            \STATE $\alpha = \argmin_{\alpha} \frac{1}{m} \sum_{i=1}^{m} \biggl \lVert f(x_{i}) - \sum_{j=1}^{k} \alpha_{j} v_{j}(x_{i} \mid \theta_{j}) \biggr\rVert_2^{2}$
            \STATE $\hat{\mathcal{T}}f = \sum_i \sigma_i \alpha_i u_i$
            \STATE $L \mathrel{+}= \frac{1}{Np}\sum_{i=1}^p \biggl\lVert \mathcal{T}f(y_i) - \hat{\mathcal{T}}f(y_i) \biggr\rVert_2^2$ 
        \ENDFOR
        \STATE \specialcomment{\# Back propagate to all parameters using gradient descent}
        \FOR{$j=1...k$}
            \STATE $\theta_j = \theta_j - \eta \nabla_{\theta_j} L$
            \STATE $\sigma_j = \sigma_j - \eta \nabla_{\sigma_j} L$
            \STATE $\vartheta_j = \vartheta_j - \eta \nabla_{\vartheta_j} L$

        \ENDFOR
        
    \ENDWHILE
    \STATE \textbf{return} $\lbrace u_j \rbrace, \lbrace \sigma_j \rbrace, \lbrace v_j \rbrace$ \specialcomment{\# i.e.\ $U \Sigma V^T$}
  \end{algorithmic}
\end{algorithm}

\begin{algorithm}[!h]
  \caption{ED Extension of the Function Encoder Algorithm}
  \fontsize{10.0pt}{10.0pt}\selectfont
  \label{alg:ed}
  \begin{algorithmic}[1]

    \STATE \textbf{Input:} Step size $\eta$, set of data sets $\mathcal{D}$
    \STATE Initialize $\{v_j\}$ parameterized by $\{\theta_j\}$ 
    \STATE Initialize $\{\lambda_j\}$ 

    \WHILE{not converged}
        \STATE \specialcomment{\# Compute loss for all pairs $(f, \mathcal{T}f)$}
        \STATE Loss $L = 0$
        \FOR{$\{\big (x_i, f_n(x_i) \big )\}_{i=1}^m, \{\big (y_i, \mathcal{T}f_n(y_i) \big )\}_{i=1}^p $ in $\mathcal{D}$}
            \STATE $\alpha = \argmin_{\alpha} \frac{1}{m} \sum_{i=1}^{m} \biggl \lVert f(x_{i}) - \sum_{j=1}^{k} \alpha_{j} v_{j}(x_{i} \mid \theta_{j}) \biggr\rVert_2^{2}$
            \STATE $\hat{\mathcal{T}}f = \sum_i \lambda_i \alpha_i v_i$
            \STATE $L \mathrel{+}= \frac{1}{Np}\sum_{i=1}^p \biggl\lVert \mathcal{T}f(y_i) - \hat{\mathcal{T}}f(y_i) \biggr \rVert_2^2$ 
        \ENDFOR
        \STATE \specialcomment{\# Back propagate to all parameters using gradient descent}
        \FOR{$j=1...k$}
            \STATE $\theta_j = \theta_j - \eta \nabla_{\theta_j} L$
            \STATE $\lambda_j = \lambda_j - \eta \nabla_{\lambda_j} L$
        \ENDFOR
        
    \ENDWHILE
    \STATE \textbf{return} $\lbrace \lambda_j \rbrace, \lbrace v_j \rbrace$ \specialcomment{\# i.e.\ $\Lambda, V$}
  \end{algorithmic}
\end{algorithm}

\begin{figure}[h]

\centering
\subfloat[Eigen value decay]{
\includegraphics[width=0.48\linewidth]{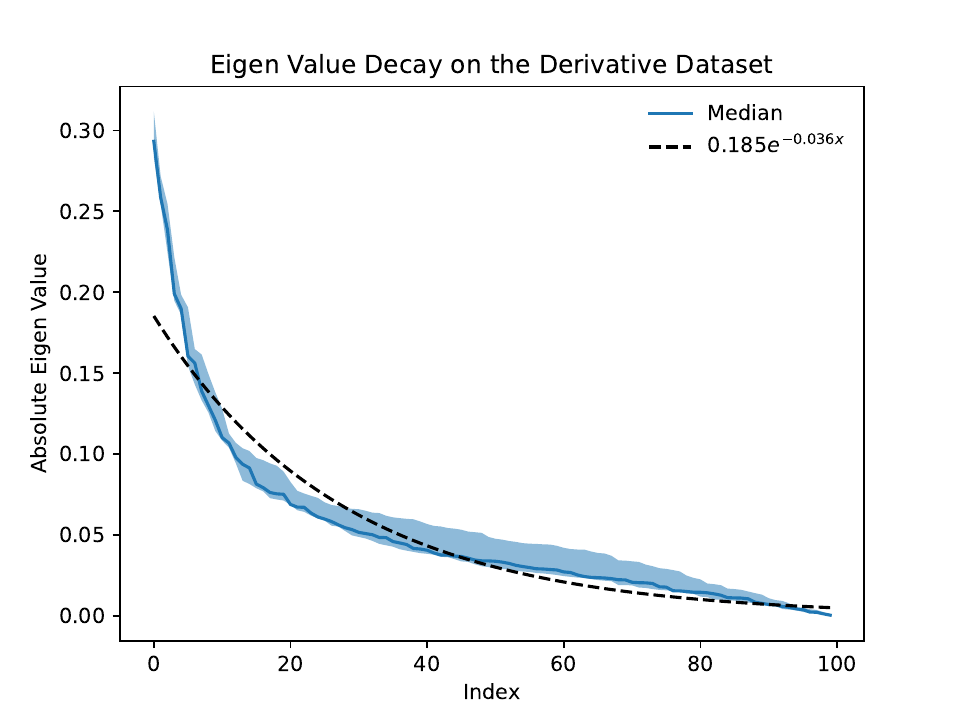}
\label{fig:derivative_eigen_decay1}
}
\hfill
\subfloat[Matrix singular value decay]{
\includegraphics[width=0.48\linewidth]{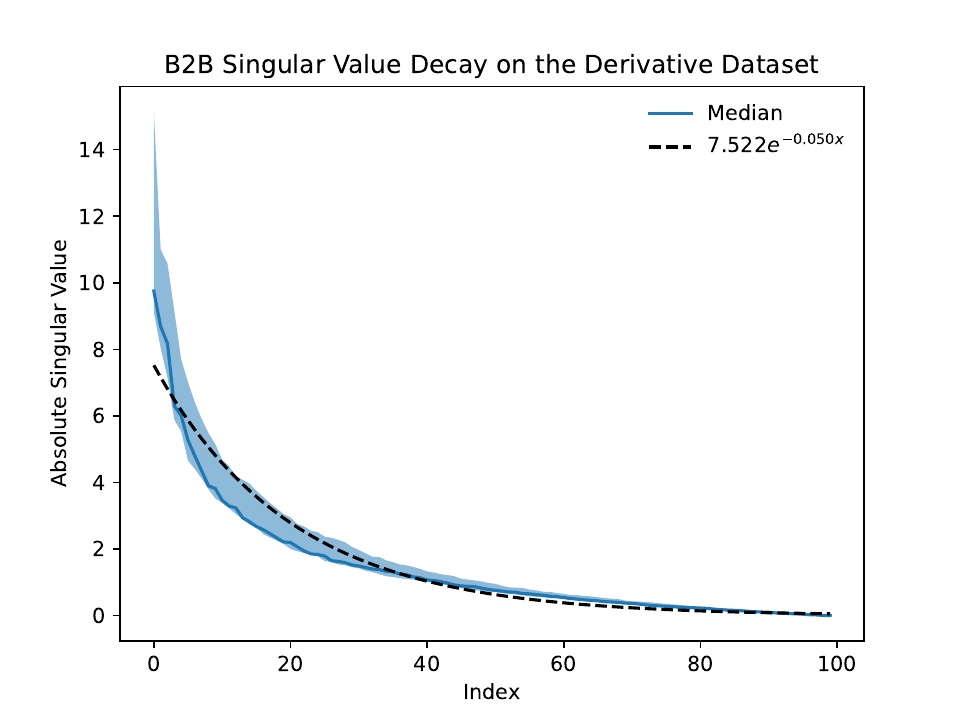}
\label{fig:derivative_eigen_decay3}
}
\caption{Decay curves for the derivative operator. (a) Shows the decay of eigenvalues from the ED approach. (b) Illustrates the decay of the singular values computed via the traditional singular value decomposition of the B2B matrix representation of the operator. Both plots demonstrate the rapid decay characteristic of compact operators.}
\label{fig:derivative_decay_comparison}
\end{figure}

%%%%%%%%%%%%%%%%%%%%%%%%%
%%%%%%%%%%%%%%%%%%%%%%%%%

\section{Results and Discussion}
\label{sec:results}

We demonstrate our proposed approach on six benchmark problems. We first show that our proposed function encoder-based operator learning approaches can efficiently learn linear operators and generalize to unseen inputs on two operator learning tasks: the derivative and anti-derivative operators. 
We then demonstrate the capabilities of our approach for modeling solutions to linear as well as nonlinear PDEs. We compare the accuracy of our approaches against one of the most popular operator learning frameworks, DeepONet, and some of its variants. 
The vanilla architecture of DeepONet requires the input functions to be sampled at fixed sensor locations across both test and train datasets. 
For comparison, we have modified our experiments accordingly to accommodate this requirement. However, we note that our approach offers greater flexibility by allowing random sampling from the input space. We considered three variants of DeepONet for comparison: 
\begin{itemize}
    \item \textbf{Vanilla DeepONet} \cite{lu2022comprehensive} - We considered the unstacked version of DeepONet. For the branch network of DeepONet, we have considered either a convolutional neural network (CNN) or a multi-layer perception, depending on the problem.
    \item \textbf{POD-DeepONet} \cite{kontolati2023influence} - 
    In this architecture, the trunk network basis is precomputed by performing proper orthogonal decomposition (POD) on the training data after first removing the mean. The POD basis is used as a non-trainable trunk network and only employs DNNs for the branch net to learn the coefficients of the POD basis. The discretization of the solution operator for the training and the testing samples must be the same. The solution operator cannot be interpolated over the domain.
    \item \textbf{Two-Stage DeepONet} \cite{lee2024training} - This framework involves a two-stage training in optimizing the trainable parameters of DeepONet. Instead of simultaneously training both the branch and trunk networks, which typically results in a high-dimensional, nonconvex optimization problem, this framework adopts a sequential strategy. First, it trains the trunk network independently. In the original formulation, the trunk network is trained to span the output function space as basis functions, where the coefficients are simultaneously learned through gradient descent. However, this is not amenable to batch training because the coefficients of \textit{all} functions must be updated at every gradient step. We slightly modify this method to compute the coefficients via least squares, as in \cite{ingebrand2024zeroshottransferneuralodes}, which is thus amenable to batch training. 
    Second, it proceeds to train the branch network, effectively decomposing the complex training task into two more manageable subtasks. Precisely, this is using the theory of function encoders to learn basis functions over the output function space, but it still learns the operator from the input data without formulating the input data as a function space, as in DeepONet. This approach is further enhanced by incorporating the Gram-Schmidt orthonormalization process, which significantly improves the stability and generalization ability of the network.
\end{itemize}

\noindent All approaches use approximately 500,000 trainable parameters. Training assumptions may be slightly different, \textit{e.g.}, B2B and two-stage DeepONet train on individual function spaces, and POD-DeepONet requires fixed output locations. All algorithms are evaluated through the same held-out data and loss function. The mean and standard deviation of the MSE over 10 independent training trials are reported in Table~\ref{tab:results}. 
The details of the training and testing sample split are reported in the appendix. All graphs show the median, first, and third quartiles over $10$ seeds.

\begin{table}[h]
    \caption{Mean squared error obtained for all applications presented in this work using Function Encoders and DeepONet.}
    \fontsize{5.9pt}{5.9pt}\selectfont
    \label{tab:results}
    \centering
    \begin{tabular}{lcccccc} %|l|c|c|c|c|c|c|c|c|} 
        \toprule
        \multicolumn{1}{c}{Dataset} & \multicolumn{3}{c}{Function Encoders} & \multicolumn{3}{c}{DeepONet} \\
        \cmidrule{1-7}
            & B2B & SVD & Eigen & Vanilla & POD & Two-stage\\
        \midrule
        Anti-Derivative & $\mathbf{1.06\mathrm{e}{-02} \pm 1.62\mathrm{e}{-02}}$ & $1.31\mathrm{e}{+00} \pm 1.04\mathrm{e}{+00}$  & $2.02\mathrm{e}{+00} \pm 2.63\mathrm{e}{+00}$ & $4.48\mathrm{e}{-01} \pm 2.14\mathrm{e}{-01}$ & $1.96\mathrm{e}{+03} \pm 1.34\mathrm{e}{+02}$ & $2.20\mathrm{e}{-01} \pm 7.95\mathrm{e}{-02}$ \\
        %\midrule
        Derivative  & $\mathbf{8.63\mathrm{e}{-04} \pm 6.60\mathrm{e}{-04}}$ & $3.33\mathrm{e}{-02} \pm 2.03\mathrm{e}{-02}$ &  $4.05\mathrm{e}{-03} \pm 3.45\mathrm{e}{-03}$ & $3.68\mathrm{e}{-03} \pm 2.57\mathrm{e}{-03}$ & $9.84\mathrm{e}{+00} \pm 6.27\mathrm{e}{-01}$ & $2.33\mathrm{e}{-03} \pm 1.01\mathrm{e}{-03}$ \\
        %\midrule
        1D Darcy Flow & $\mathbf{1.74\mathrm{e}{-05} \pm 4.92\mathrm{e}{-06}}$ & $8.90\mathrm{e}{-04} \pm 8.03\mathrm{e}{-05}$ &  -  & $4.47\mathrm{e}{-05} \pm 8.94\mathrm{e}{-06}$ & $3.35\mathrm{e}{-05} \pm 8.79\mathrm{e}{-06}$ & $2.59\mathrm{e}{-04} \pm 8.43\mathrm{e}{-05}$ \\
        %\midrule
        2D Darcy Flow  & $\mathbf{5.30\mathrm{e}{-03} \pm 1.19\mathrm{e}{-03}}$ & $2.89\mathrm{e}{-02} \pm 2.31\mathrm{e}{-03}$ &  -  & $2.68\mathrm{e}{-02} \pm 2.77\mathrm{e}{-03}$ & $2.50\mathrm{e}{-02} \pm 1.64\mathrm{e}{-03}$ & $1.33\mathrm{e}{-02} \pm 1.55\mathrm{e}{-03}$\\
        %\midrule
        Elastic Plate & $\mathbf{6.30\mathrm{e}{-05} \pm 5.59\mathrm{e}{-05}}$  & $1.03\mathrm{e}{-01} \pm 1.83\mathrm{e}{-02}$  &  -  & $4.66\mathrm{e}{-04} \pm 8.16\mathrm{e}{-04}$ & $5.59\mathrm{e}{-04} \pm 1.15\mathrm{e}{-03}$ &  -  \\
        %\midrule
        Parameterized Heat Equation  & $\mathbf{4.07\mathrm{e}{-04} \pm 2.86\mathrm{e}{-04}}^*$ & $2.27\mathrm{e}{-01} \pm 2.35\mathrm{e}{-02}$ &  -  & $6.00\mathrm{e}{-04} \pm 1.09\mathrm{e}{-03}$ & $8.88\mathrm{e}{-01} \pm 1.15\mathrm{e}{-01}$ &  -  \\
        Burger's Equation & $\mathbf{5.07\mathrm{e}{-04} \pm 1.93\mathrm{e}{-04}}$ & $1.01\mathrm{e}{-01} \pm 1.16\mathrm{e}{-02}$ & -  & $2.16\mathrm{e}{-03} \pm 5.59\mathrm{e}{-04}$ & $1.94\mathrm{e}{+00} \pm 1.76\mathrm{e}{-01}$ & $2.03\mathrm{e}{+00} \pm 1.78\mathrm{e}{-01}$ \\
        \bottomrule
    \end{tabular}
    \caption*{*While the mean of prediction errors for B2B is lower than DeepONet for the parameterized heat equation dataset, the median is higher as \\ shown in Figure \ref{fig:heat}.}

\end{table}

\subsection{Derivative and Anti-Derivative Operators}
\label{subsec:example1}

To demonstrate our approach, we first considered a pedagogical problem of learning the anti-derivative operator for a class of functions. Consider a 1D problem, defined as:
\begin{equation}
    \frac{ds(x)}{dx} = u(x), \quad x \in [-10,10], \quad s(0) = 0,
\end{equation}
where our aim is to compute the anti-derivative operator $\mathcal{T}$, such that
\begin{equation}
    \mathcal{T} u (x) = s(x=0) + \int_{0}^{x} u(t) dt, 
\end{equation}
where $u(x)$ is a quadratic polynomial. The anti-derivative operator aims to learn the mapping from the space of quadratic polynomials (degree two) to cubic polynomials (degree three). Using B2B, we first compute two sets of basis functions, $g_{1}, \ldots, g_{k}$ to span the input function space and $h_{1}, \ldots, h_{\ell}$ to span the output function space. We then calculate a transformation that maps the coefficients of the input basis functions to those of the output basis functions. Given the linearity of the anti-derivative operator, this transformation is represented by a matrix, which we obtain by solving the least-squares problem in \eqref{eqn: linear operator least squares}. 

We also considered the derivative operator, where the input function space consists of cubic polynomials (degree three), while the output function space comprises quadratic polynomials (degree two). As with the anti-derivative operator, the derivative operator is linear, allowing us to compute the matrix transformation through a least-squares optimization as in \eqref{eqn: linear operator least squares}. Qualitative results showing the \emph{worst-case} evaluation from our dataset are presented in Figure \ref{fig:qualitative_linear}. Quantitative results are presented in Figures \ref{fig:integral} and \ref{fig:derivative}, respectively.

\begin{figure}[h]
    \begin{minipage}{.48\linewidth}
        \includegraphics[width=1.0\linewidth]{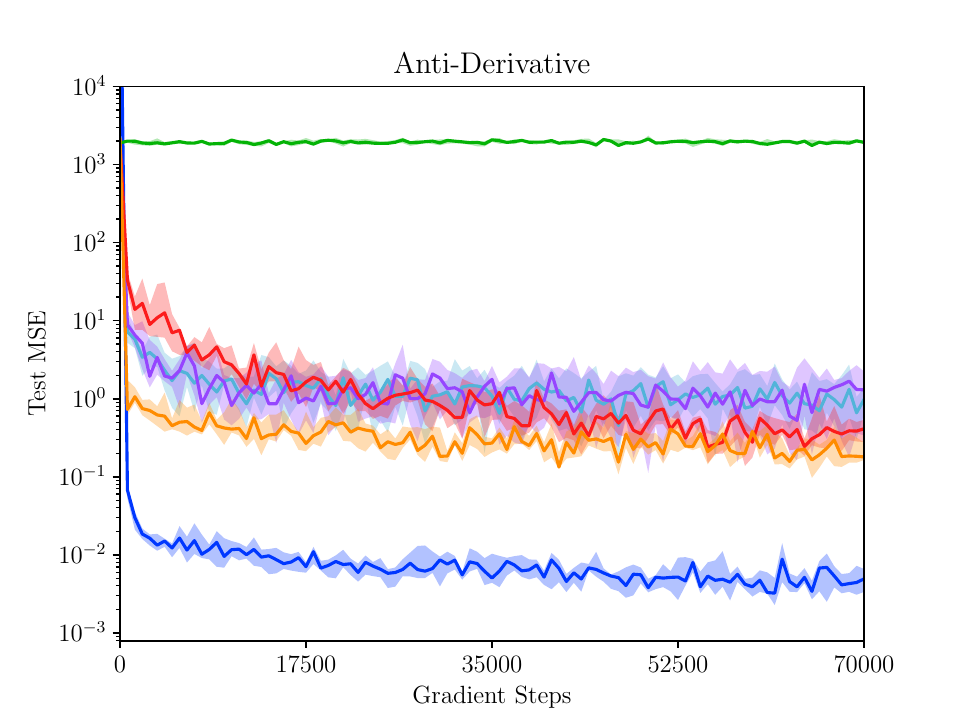}
        \caption{\textit{Training curves on the anti-derivative dataset.} This figure plots the test MSE for each algorithm during training. Notably, B2B achieves orders of magnitude better performance than the baselines. }
        \label{fig:integral}
    \end{minipage}%
    \hfill
    \begin{minipage}{.48\linewidth}
        \includegraphics[width=1.0\linewidth]{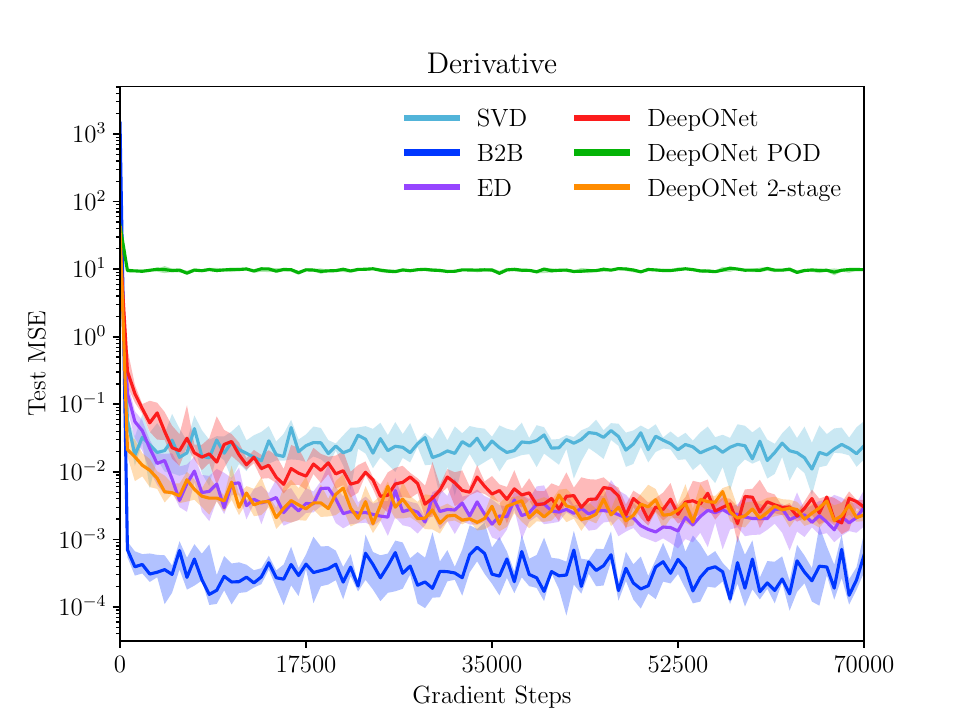}
        \caption{\textit{Training curves on the derivative dataset.} This figure plots the test MSE for each algorithm during training. B2B achieves an order of magnitude better performance than the baselines. }
        \label{fig:derivative}
    \end{minipage}%
\end{figure}

\begin{figure}[h!]
    \centering
    \subfloat[\textit{Worst-case sample for the anti-derivative dataset.} This figure shows an example of a function $f \in \mathcal{G}$ and the corresponding $\mathcal{T}f$ for the anti-derivative dataset. The function is chosen to have the worst-case performance for B2B. Left: $f$ and B2B's approximation $\hat{f}$. B2B can estimate $f$ from $D$ because it has learned basis functions over the input function space. Center-left: the absolute error $|f-\hat{f}|$. Center-right: $\mathcal{T}f$ and the approximations for both B2B and DeepONet. Right: the absolute error $|\mathcal{T}f - \hat{\mathcal{T}}f|$ for both approaches. ]{
    \includegraphics[width=\linewidth]{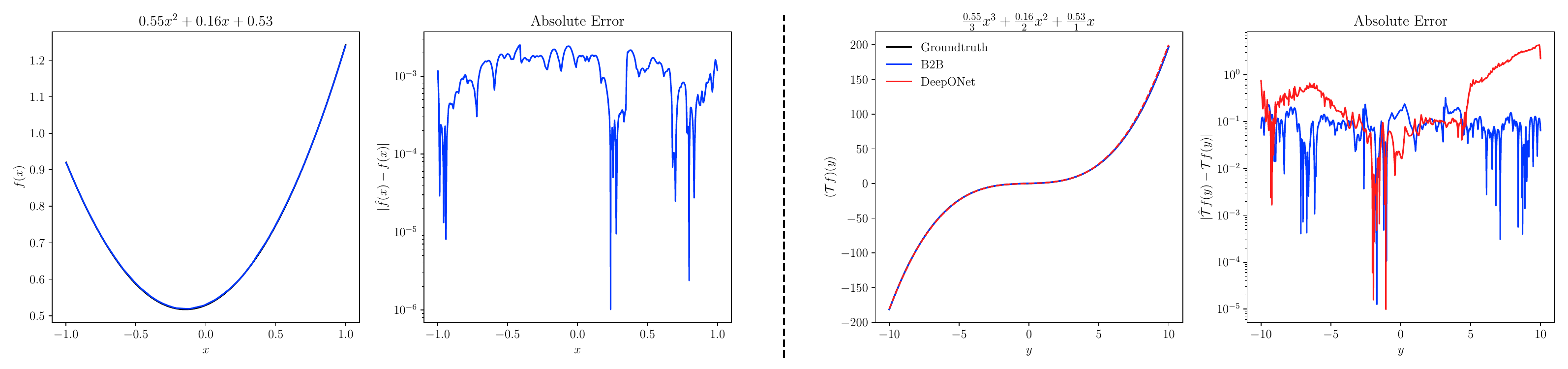}
    % \caption{}
    % \label{fig:anti-derivative}
    }
    \vspace{1em} 
    \subfloat[\textit{Worst-case sample for the derivative dataset.} This figure shows an example of a function $f \in \mathcal{G}$ and the corresponding $\mathcal{T}f$ for the derivative dataset. The function is chosen to have the worst-case performance for B2B. Left: $f$ and B2B's approximation $\hat{f}$. B2B can estimate $f$ from $D$ because it has learned basis functions over the input function space. Center-left: the absolute error $|f-\hat{f}|$. Center-right: $\mathcal{T}f$ and the approximations for both B2B and DeepONet. Right: the absolute error $|\mathcal{T}f - \hat{\mathcal{T}}f|$ for both approaches. ]{
    \includegraphics[width=1.0\linewidth]{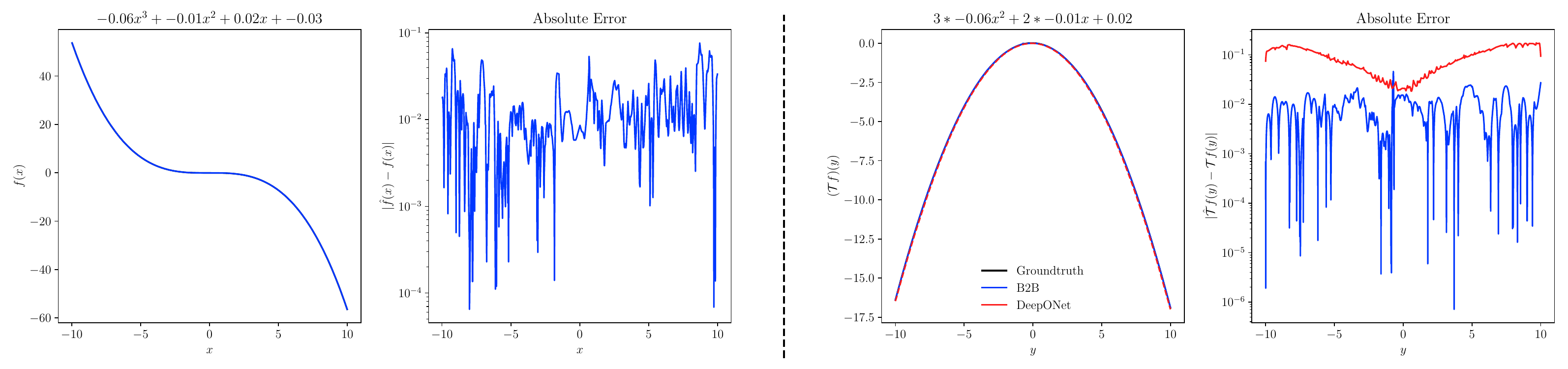}}
    \caption{\textit{Qualitative comparison between B2B and DeepONet.}}
    \label{fig:qualitative_linear}
\end{figure}

Our analysis of the anti-derivative and the derivative examples demonstrates that B2B significantly outperforms the baselines in both convergence speed and asymptotic performance. This stems from B2B's ability to exploit the linear structure of the operator by directly computing the transformation matrix via least-squares optimization.
SVD and ED also demonstrate comparable performance with the baseline approaches on these linear examples, as they are inherently designed to learn linear operators. However, their convergence is slower compared to B2B because they must learn properly aligned basis functions rather than directly computing the operator. 
Overall, these approaches effectively capitalize on the operator's linearity, enabling an efficient and precise computation of the transformation between input and output function spaces.

In contrast, vanilla DeepONet exhibits less accurate asymptotic performance and slower convergence for learning the derivative and anti-derivative operators. While two-stage DeepONet shows improved results, the error on the derivative operator dataset is an order of magnitude higher than B2B, which we suspect is due to its inability to fully exploit the linear structure of the operator. For the anti-derivative operator, the accuracy of B2B is two orders of magnitude higher than competing approaches. 
Notably, POD-DeepONet completely fails in these examples, and we can see that the MSE on the test dataset is two to three orders of magnitude higher than all other approaches.

Next, to demonstrate the robustness of B2B, we conduct out-of-distribution (OOD), linearity, and homogeneity tests on the derivative and the anti-derivative operators. 
For the OOD test, we sample polynomials that are much larger in magnitude than anything seen during training, but still of the same degree. 
Since the operator is linear, the algorithms should be able to generalize to the entire linear span of the input space.
For the linearity test, we conduct experiments to check if the operator commutes with a linear combination, defined as: $\mathcal{T}(af(x)+bg(x)) = a\mathcal{T}f(x) + b\mathcal{T}g(x)$, where $a$ and $b$ are random scalars, and $f(x)$ and $g(x)$ are polynomial functions. 
In other words, we compute the coefficients of the output space basis functions using data from $f$ and $g$ and then test if a linear combination of these coefficients can reproduce the operator action on the linear combination of $f$ and $g$. 
The homogeneity test is analogous with $\mathcal{T}(af(x)) = a\mathcal{T}f(x)$. We perform this analysis on B2B and vanilla DeepONet. The results, presented in Table~\ref{tab:combined_results} demonstrate that B2B outperforms DeepONet by one to two orders of magnitude in the OOD test. One example of the OOD test is shown in Figure~\ref{fig:linearity_test}.  Furthermore, DeepONet fails on the linearity and homogeneity tests, meaning the test error increases by several orders of magnitude, indicating that it is unable to generalize well to the linear span of the data. We suspect this is largely because the architecture of DeepONet and its variants is not designed to exploit the linear structure of the operator. In contrast, from Theorem \ref{theorem}, B2B has guaranteed generalization with the linear span of the input function space for linear operators.

\begin{figure}[h]
    \centering

    \includegraphics[width=1.0\linewidth]{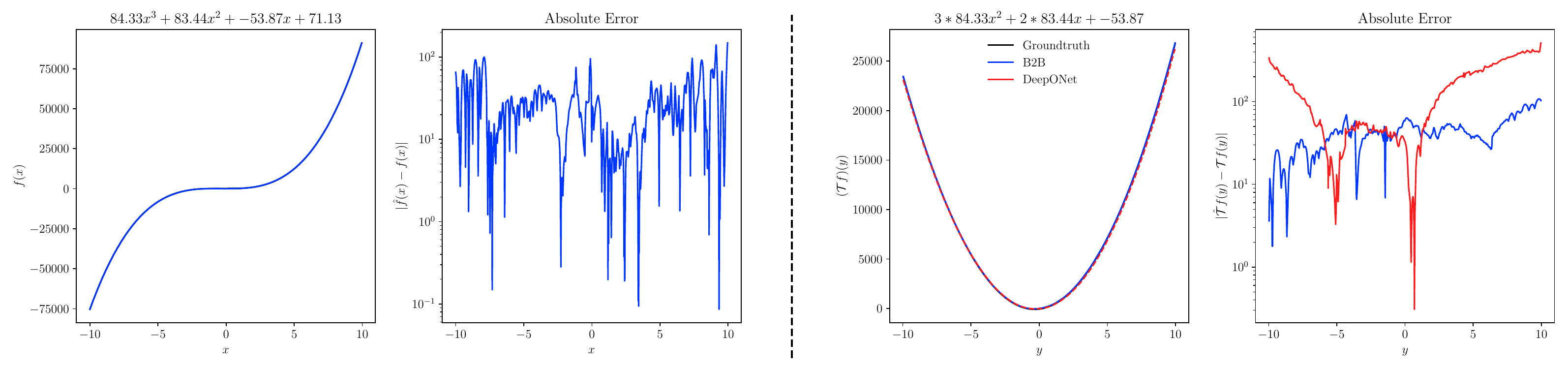}
    \caption{\textit{Worst-case sample for the OOD derivative test.} This figure shows an example of a function $f \in \mathcal{G}$ and the corresponding $\mathcal{T}f$ for the derivative dataset. The function is chosen to have the worst-case performance for B2B, and the function is far outside of the training set. Left: $f$ and B2B's approximation $\hat{f}$. Center-left: the absolute error $|f-\hat{f}|$. Center-right: $\mathcal{T}f$ and the approximations for both B2B and DeepONet. Right: the absolute error $|\mathcal{T}f - \hat{\mathcal{T}}f|$ for both approaches. }
    \label{fig:linearity_test}
\end{figure}

\begin{table}[h]
    \caption{Out-of-distribution, linearity, and homogeneity analysis on the derivative and anti-derivative datasets.}
    \label{tab:combined_results}
    \centering
    % \tiny
    \begin{tabular}{lccc} 
        \toprule
        \multicolumn{4}{c}{Derivative Dataset} \\
        \midrule
        & OOD Function MSE & Linearity Function MSE & Homogeneity Function MSE \\
        \midrule
        B2B & $\mathbf{6.22\mathrm{e}{+02} \pm 1.28\mathrm{e}{+03}}$ & $\mathbf{3.82\mathrm{e}{-03} \pm 1.08\mathrm{e}{-02}}$ & $\mathbf{1.89\mathrm{e}{-03} \pm 5.50\mathrm{e}{-03}}$ \\
        DeepONet & $9.04\mathrm{e}{+03} \pm 2.82\mathrm{e}{+04}$ & $5.83\mathrm{e}{+02} \pm 1.71\mathrm{e}{+03}$ & $2.90\mathrm{e}{+02} \pm 9.35\mathrm{e}{+02}$ \\
        \midrule
        \multicolumn{4}{c}{Anti-Derivative Dataset} \\
        \midrule
        & OOD Function MSE & Linearity Function MSE & Homogeneity Function MSE \\
        B2B & $\mathbf{7.50\mathrm{e}{-03} \pm 2.51\mathrm{e}{-02}}$ & $\mathbf{1.24\mathrm{e}{-01} \pm 5.24\mathrm{e}{-01}}$ &  $\mathbf{6.16\mathrm{e}{-02} \pm 2.75\mathrm{e}{-01}}$ \\
        DeepONet & $4.58\mathrm{e}{-01} \pm 1.60\mathrm{e}{+00}$ & $3.88\mathrm{e}{+04} \pm 1.30\mathrm{e}{+05}$ & $1.94\mathrm{e}{+04} \pm 6.88\mathrm{e}{+04}$ \\
        \bottomrule
    \end{tabular}
\end{table}

We also visualize the loss landscape of the gradients of the function spaces to gain insights into the optimization behavior of our approach versus DeepONet. We first sample 100 functions from the input space to generate the loss landscape. For each function, we compute the loss gradient with respect to the neural network parameters. We then construct a $100 \times P$ matrix from these gradients, where $P$ is the total number of parameters. Next, we perform Principal Component Analysis (PCA) to identify the two most significant gradient directions. We normalize these directions based on the magnitude of the coefficients for each layer to account for the scale invariance of ReLU activations. Finally, we compute the loss for models with parameters $\theta = \theta^* + \alpha p_1 + \beta p_2$, where $\theta^*$ is the optimal parameter set, $p_1$ and $p_2$ are the principal component vectors, and $\alpha$ and $\beta$ range from -0.01 to 0.01. For more detail on this procedure, see \cite{li2018visualizinglosslandscapeneural}. See Figure~\ref{fig:loss-landscape-derivative} for a visualization of the loss landscape for the derivative dataset.

B2B demonstrates smooth and well-behaved loss landscapes for both the input and output function spaces. These landscapes exhibit gradual, continuous slopes with a clear bowl-like structure, indicating an approximately convex optimization problem with stable gradients during training. 
In contrast, the DeepONet loss landscape displays a sharp, spike-like feature at its center, suggesting potential instabilities in the optimization process. The sharp gradient at the center could lead to challenges in convergence, as small perturbations in the parameters might result in large changes in the loss value. This characteristic may make the optimization process more sensitive to the learning rate and initial conditions, potentially requiring careful tuning to achieve good performance.

The difference between these loss landscapes highlights a key advantage of our approach. 
The results suggest that decomposing the operator learning problem into two separate function spaces yields better behaved optimization problems than approximating the operator end-to-end. 
This difference likely explains the improved convergence speed of B2B. 

\begin{figure}[h!]
    \centering
    \includegraphics[width=\linewidth]{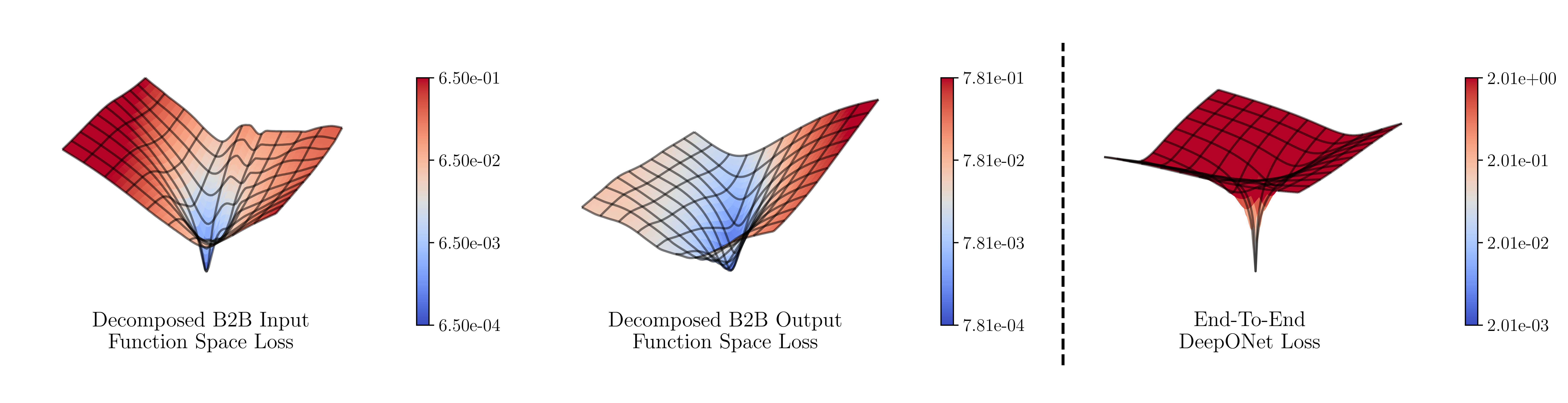}
    \caption{\textit{Visualization of the loss landscape in two main PCA directions on the derivative dataset}. The two figures show the decomposed loss landscapes for the input and output function spaces, respectively, for the B2B algorithm. The right figure shows the loss landscape for the full end-to-end loss of DeepONet. Each color bar ranges from the locally optimal loss to 1000x this amount. Therefore, the dark red regions indicate a 1000x drop in performance. The sharp loss landscape and abrupt color change for DeepONet therefore suggests that a small change in parameters leads to a drastic decrease in performance, highlighting the sensitivity of this approach. In contrast, both decomposed B2B losses have larger blue regions with a smoother shape, indicating relative stability in the training procedure.   }
    \label{fig:loss-landscape-derivative}
\end{figure}

\begin{figure}[h!]
    \centering
    \includegraphics[width=0.75\linewidth]{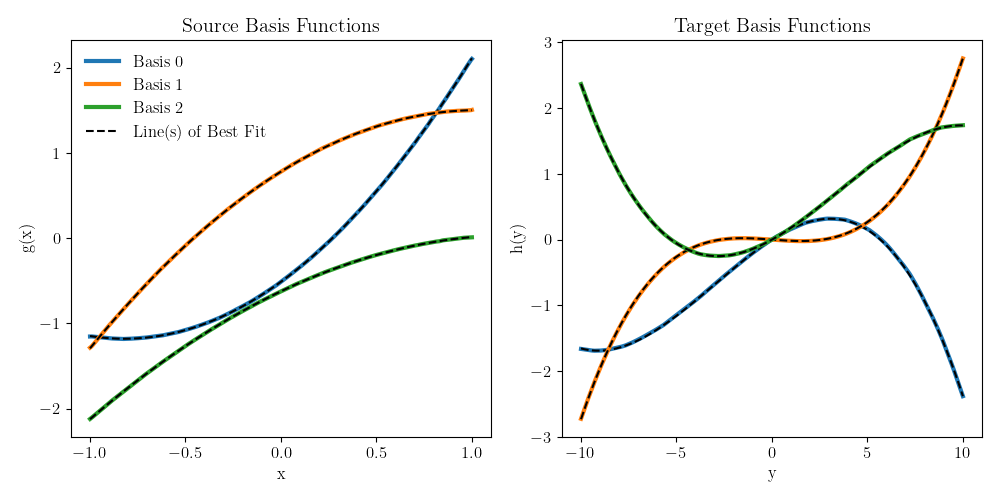}
    \caption{\textit{Visualization of a set of learned basis functions for the anti-derivative dataset.} The input function space is quadratics, and we therefore expect the basis functions on the left to be quadratic. On the right, we expect cubic functions with a 0 for the constant term. We also plot the line of best fit for each basis function, where the line of best fit is constrained to the desired class. The tight overlap between each basis function and its fit suggests that the learned basis functions are indeed quadratic and cubic, respectively. }
    \label{fig:basis_functions}
\end{figure}

Lastly, we may visualize the basis functions on these simple examples to determine if the learned functions align with our expectations. In the anti-derivative dataset, the input function space is the three-dimensional space of quadratic polynomials, and the output function space is the three-dimensional subspace of cubic polynomials with 0 for the constant term. If we use a properly specified basis for each space, i.e. three basis functions, we would expect each basis function to lie in the respective spaces. We plot an example of learned basis functions in Figure \ref{fig:basis_functions}. The results indicate that learned basis functions do indeed converge to the space; Note this is the case because the number of basis functions is well-specified. If we use an overspecified basis, i.e. more than three basis functions, each individual basis function may not lie within the space while, collectively, they span the space. Nonetheless, an overspecified basis will achieve the same performance, so its advisable to err on the side of caution and use a overly large number of basis functions.

\subsection{1D Darcy Flow}
\label{subsec:example2}

In this example, we aim to learn the nonlinear Darcy operator for a 1D system. A variant of the nonlinear 1D Darcy's equation is written as:
\begin{equation}
    \frac{ds}{dx}
    \left(
    -\kappa(s(x)) \frac{d s}{dx}
    \right)
     = u(x), \quad x\in[0, 1], 
\end{equation}
where the solution-dependent permeability is $\kappa(s(x)) = 0.2 + s^2(x)$ and the input term is a Gaussian random field $s(x)\sim \mathcal{GP}$ defined as $s(x) \sim \mathrm{GP}(0, k(x, x'))$ such that $k(x, x') = \sigma^2 \exp(- {\|x - x'\|^2}/{(2\ell_x^2)})$, where $\ell_{x} = 0.04, \ \sigma^2 = 1.0$. Homogeneous Dirichlet boundary conditions $s=0$ are considered at the domain boundaries. 
% \textcolor{red}{This problem has been adapted from \cite{bahmani2024resolution} but was generated using a finite differences method}.
We adapt this problem from \cite{bahmani2024resolution}, using a new dataset generated via the finite difference method.
The training data uses fixed sampling points for DeepONet compatibility, though our proposed approaches support arbitrary sampling locations.

The convergence of the test error over gradient steps is presented in Figure~\ref{fig:darcy}. B2B achieves the lowest test error and the fastest convergence among all the tested approaches. The accuracy of B2B and vanilla DeepONet are shown for one sample in Figure~\ref{fig:1dDarcy}.

\begin{figure}[!b]
    \begin{minipage}[t]{.48\linewidth}
    \includegraphics[width=1.0\linewidth]{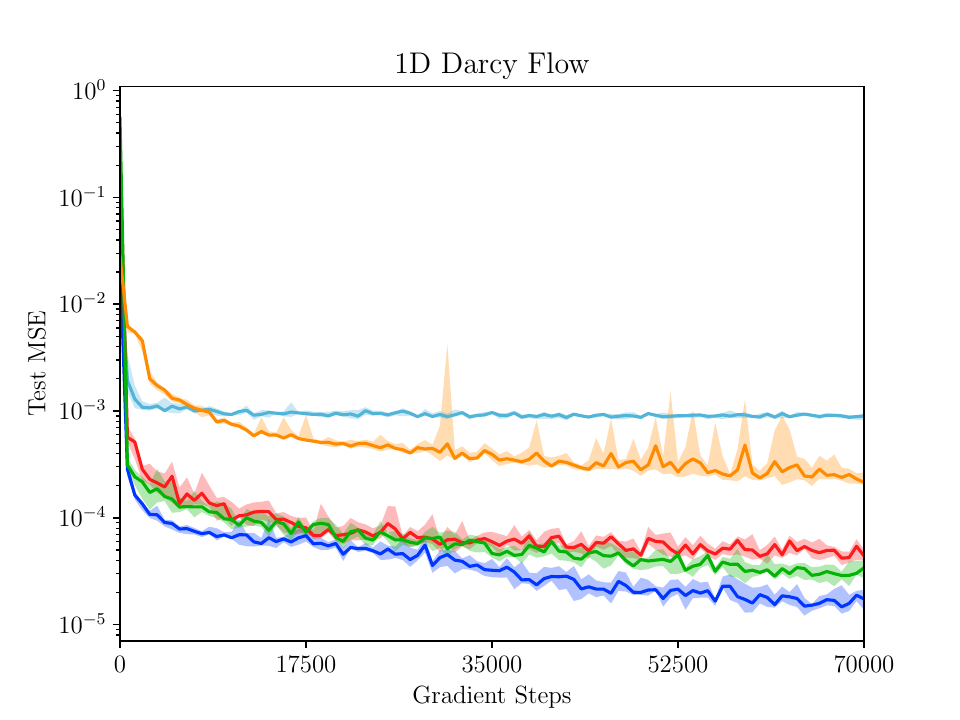}
    \caption{\textit{Training curves on the 1D Darcy flow dataset.}  This figure plots the test MSE for each algorithm during training. B2B achieves slightly better convergence speed and performance relative to other approaches. }
    
    \label{fig:darcy}
    \end{minipage}%
    \hfill
    \begin{minipage}[t]{.48\linewidth}       \includegraphics[width=1.0\linewidth]{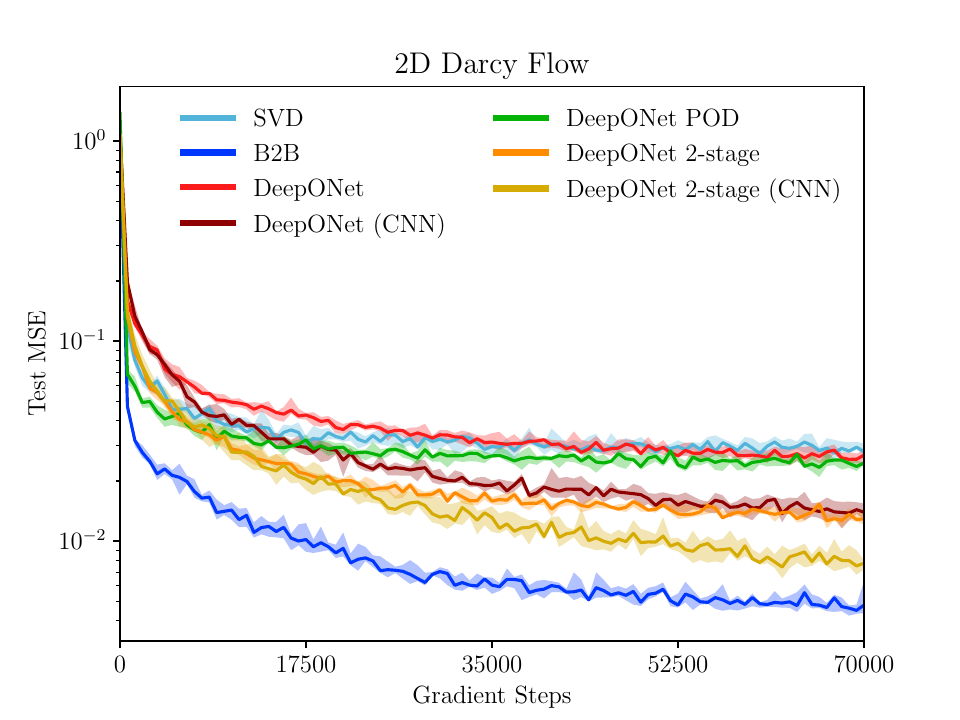}
    \caption{\textit{Training curves on the 2D, L-shaped Darcy flow dataset.}  This figure plots the test MSE for each algorithm during training. B2B achieves better convergence speed and performance on this dataset. Interestingly, DeepONet achieves a similar performance to SVD, even though this operator is nonlinear and SVD can, at best, learn a linear approximation. }
    \label{fig:lshaped}
    \end{minipage}%
\end{figure}

\begin{figure}[h]
    \centering
    \includegraphics[width=1.0\linewidth]{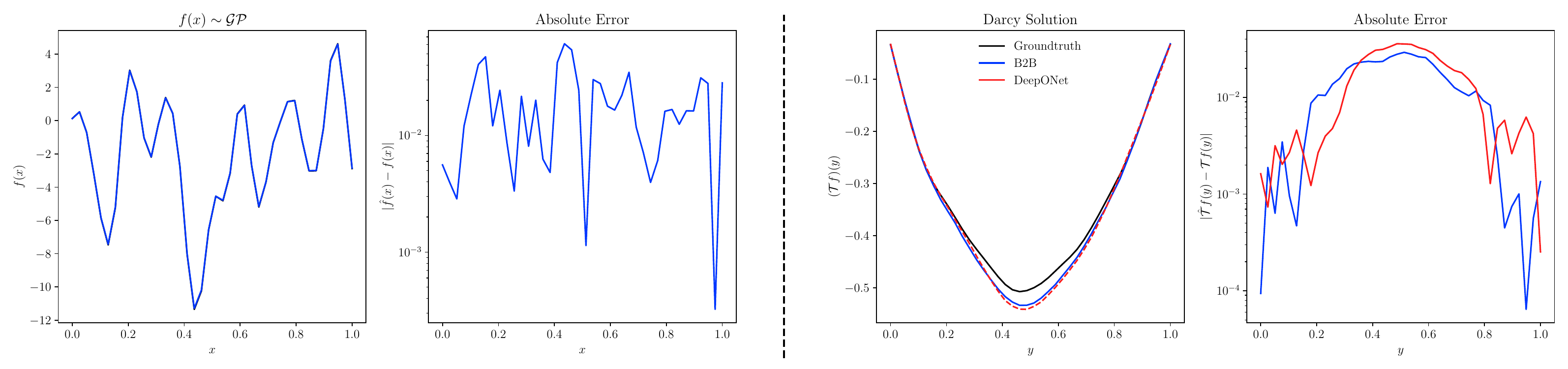}

    \caption{\textit{Worst-case sample for the 1D Darcy flow dataset.} This figure shows an example of a function $f \in \mathcal{G}$ and the corresponding $\mathcal{T}f$ for the 1D Darcy flow dataset. The function is chosen to have the worst-case performance for B2B. Left: $f$ and B2B's approximation $\hat{f}$. Center-left: the absolute error $|f-\hat{f}|$. Center-right: $\mathcal{T}f$ and the approximations for both B2B and DeepONet. Right: the absolute error $|\mathcal{T}f - \hat{\mathcal{T}}f|$ for both approaches. }
    \label{fig:1dDarcy}
\end{figure}

\begin{figure}[h]
    \centering
    \subfloat[Basis-to-Basis]{
    \includegraphics[width=\linewidth]{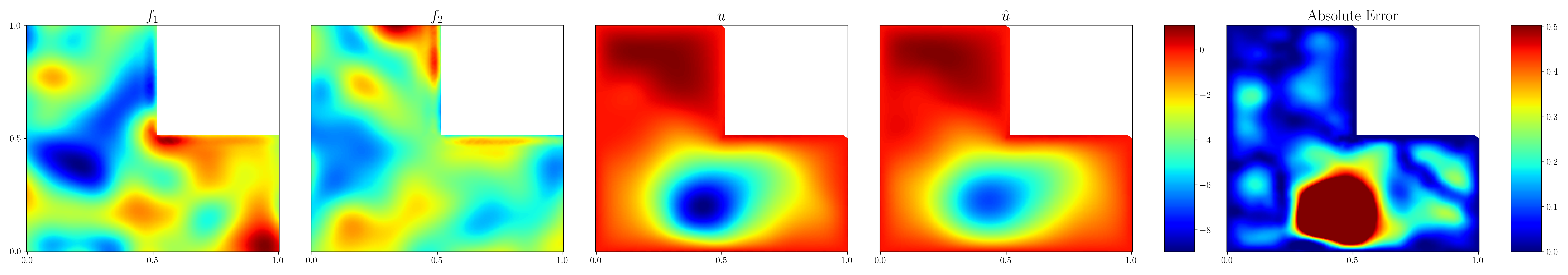}}
    \vspace{1em} 
    \subfloat[DeepONet]{
    \includegraphics[width=\linewidth]{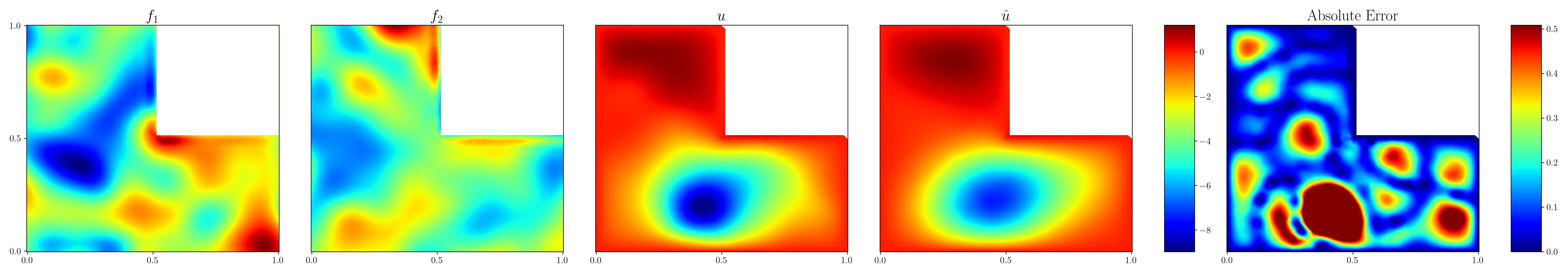}}
    \caption{ \textit{Worst-case sample for the 2D Darcy flow dataset.} This figure compares B2B (top) and DeepONet (bottom) on the worst-case sample for  B2B from the 2D Darcy flow dataset.  Left and center-left: two dimensions of the input function. Center: the ground truth output field. Center-right: the corresponding approximation. Right: the absolute error between the estimated and ground truth output. The color bar ranges from $0$ to $5\%$ of the maximum difference in output value, with dark red indicating $>\!5\%$ error.}
    \label{fig:2dDarcy}
\end{figure}

\subsection{2D Darcy Flow}
\label{subsec:example3}

In this example, we consider a Darcy flow on a 2D L-shaped domain to demonstrate the effectiveness of our approach on an irregularly-shaped domain. The problem is defined as:
\begin{align}
    \label{eqn:poisson_PDE}
    \nabla\cdot\left(k(x)\nabla u(x)\right) + f(x) &= 0,\;\;\; x = (x,y)\in\Omega:=(0,1)^2\backslash[0.5,1)^2 \\
    u(x) &= 0, \quad x\in\partial\Omega,
\end{align}
where $k(x)$ is a spatially varying permeability field, $u(x)$ is the hydraulic head, and $f(x)$ is a spatially varying force vector. Triangular elements are used to discretize the L-shaped domain $\Omega$ to run the finite element solver to generate the labeled dataset. We aim to learn the nonlinear operator, $\mathcal{T}$ that maps the permeability field and the source vector to the flow pressure, i.e.\ $\mathcal{T}: [k(\boldsymbol{x}), f(\boldsymbol x)] \mapsto u(\boldsymbol x)$. This problem has been adopted from \cite{kahana2023geometry}. The goal of this experiment is to demonstrate the applicability of our proposed approach to learning the mapping of multiple input fields to the solution operator. 

To train DeepONet with a CNN branch network, we require the input functions to be defined on a structured grid. To account for this, we discretized the L-shaped domain on a uniform $31 \times 31$ grid and padded the top right portion (the exterior of the domain) with zeros. However, the solution space is defined on an irregular grid which is discretized with 450 nodal points. 

The convergence of the test error is shown in Figure~\ref{fig:lshaped}. Our results demonstrate that B2B significantly outperforms the baseline approaches, exhibiting both superior accuracy and lower variance in prediction errors. We explored various versions of DeepONet applicable to this dataset, all of which showed improved performance compared to the base DeepONet algorithm. Interestingly, we observed that the SVD algorithm slightly outperformed vanilla DeepONet, despite the nonlinear nature of the operator. This suggests that DeepONet may be struggling to learn the operator effectively, while the strictly linear approximation of the operator obtained via the SVD approach yields reasonable performance. Of all the DeepONet variants, two-stage DeepONet with a CNN branch network has the lowest test error. A qualitative sample depicting the predictions of DeepONet and B2B is shown in Figure~\ref{fig:2dDarcy}.

\subsection{Elastic Plate}
\label{subsec:example4}

In this example, we consider a thin rectangular plate subjected to in-plane loading that is modeled as a two-dimensional problem of plane stress elasticity. The equations are given by:
\label{subsec:elasticity}
\begin{equation}
\label{eq:elasticity}
    \nabla \cdot {\sigma} + \boldsymbol{f}(\boldsymbol x) = 0,\hspace{10pt} \boldsymbol x = (x,y),
\end{equation}
\begin{equation*}
\label{eq:elasticity-BC}
    (u, v) = 0, \hspace{10pt} \forall\;\; x=0,
\end{equation*}
where $\boldsymbol{\sigma}$ is the Cauchy stress tensor, $\boldsymbol{f}$ is the body force, $u(\boldsymbol x)$ and $v(\boldsymbol x)$ represent the $x$- and $y$-displacement, respectively. In addition, $E$, and $\nu$ represent the Young modulus and Poisson ratio of the material, respectively. The relation between stress and displacement in plane stress conditions is defined as:
\begin{equation}
    \begin{Bmatrix}
        \sigma_{xx}\\ \sigma_{yy}\\ \tau_{xy}
    \end{Bmatrix} = \frac{E}{1-\nu^2}
    \begin{bmatrix}
        1 & \nu & 0 \\ \nu & 1 & 0 \\ 0 & 0 & \frac{1-\nu}{2}
    \end{bmatrix} \times     
    \begin{Bmatrix}
        \frac{\partial u}{\partial x} \\ \frac{\partial v}{\partial y} \\ \frac{\partial u}{\partial y} + \frac{\partial v}{\partial x}
    \end{Bmatrix}.
\end{equation}

We model the loading conditions $f(\boldsymbol x)$ applied to the right edge of the plate as a Gaussian random field. We aim to learn the mapping $\mathcal{T} :f(\boldsymbol x)\rightarrow [ \boldsymbol u(\boldsymbol x), \boldsymbol v(\boldsymbol x)]$ from the random boundary load to the displacement field ($\boldsymbol u$: $x$-displacement and $\boldsymbol v$: $y$-displacement). The goal of this example is to demonstrate the applicability of our proposed approach to learning the mapping to the solution operator for more than one field. The dataset for this problem has been adopted from \cite{goswami2022deep}.

The test error convergence plots are shown in Figure~\ref{fig:elastic}. The plot demonstrates that the B2B operator achieves the lowest test error. For this dataset, we evaluate against DeepONet and POD-DeepONet, both of which achieve similar accuracy. 
Two-stage DeepONet cannot be evaluated fairly in this problem because its orthonormalization procedure is not easily extended to vector-valued outputs. Therefore, this would require one set of basis functions per output dimension, which would not make for a fair comparison with the other approaches.
A qualitative sample depicting the predictions of DeepONet and B2B is shown in Figure~\ref{fig:elasticity_error}.

\subsection{Parameterized Heat Equation}
\label{subsec:example5}

In this example, we consider a parameterized heat equation in a two-dimensional domain. The governing equation is given by:
\begin{equation}
    \frac{\partial T(x,y,t)}{\partial t} = \alpha\left(\frac{\partial^2 T(x,y,t)}{\partial x^2}+\frac{\partial^2 T(x,y,t)}{\partial y^2} \right) \;\; \forall \quad(x,y,t)\in[0,1]\times[0,1]\times[0,1],
\end{equation}
where $T$ denotes temperature, $x$, $y$, and $t$ denote the spatiotemporal coordinates, and $\alpha$ denotes the parameterized thermal diffusivity. We consider a two-dimensional plate with the edges set to temperature $T=0$ and within the domain, the temperature is set to randomly chosen constant temperature. Accordingly, the initial and boundary conditions are expressed as:
\begin{align}
    T(x,y,0) &= T_{0}, \quad 0<x<1, 0<y<1,\\
    T(0,y,t) = T(1,y,t) &= 0, \quad 0< y<1, 0\leq t\leq1,\\
    T(x,0,t) = T(x,1,t) &= 0, \quad 0< x<1, 0\leq t\leq1,
\end{align}
where $T_0$ is a predefined constant temperature. The goal is to map the initial temperature of the plate $T_{0}$ and the thermal diffusivity $\alpha$ to the evolving temperature field, i.e.\ $\mathcal{T} : [T_{0}, \alpha] \rightarrow T(t)$. In contrast to the other examples, this example does not involve the mapping of the input term as an input \emph{function}. Instead, the output solution space is parameterized by $\alpha$ and $T_0$. To generate a labeled dataset, the initial temperature $T_{0}$ is considered in the range $[0,1]$, while the thermal diffusivity spans several magnitudes within $[10^{-2}, 10^{0}]$. The dataset is adapted from \cite{mandl2024separable}.

\begin{figure}[!h]
    \begin{minipage}[t]{.48\linewidth}
        % \centering
        \includegraphics[width=1.0\linewidth]{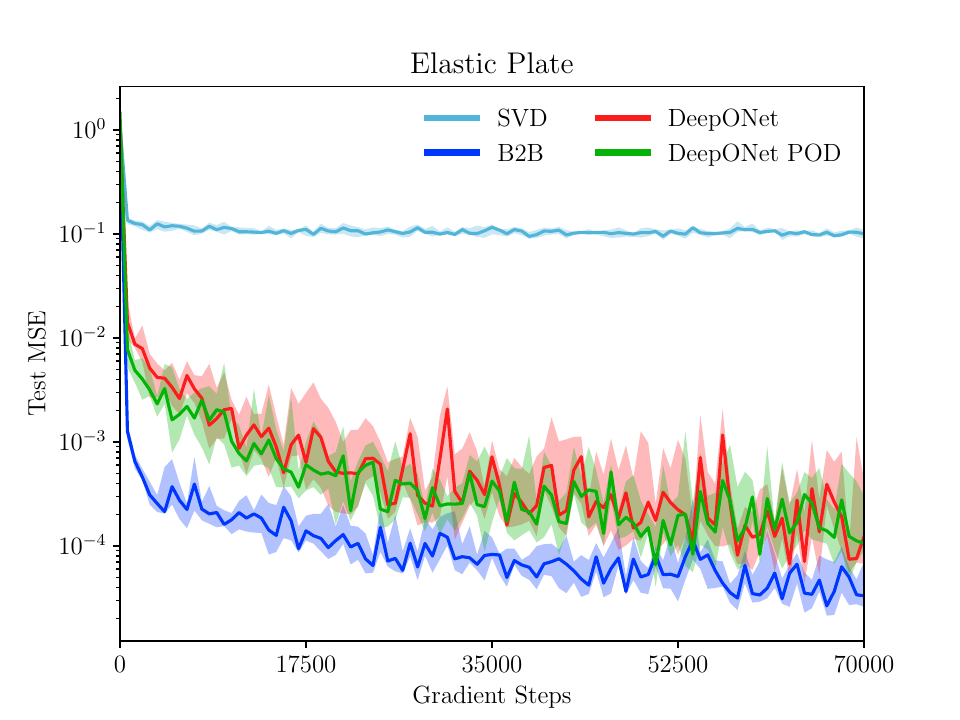}
        \caption{\textit{Training curves on the elastic plate dataset.} This figure plots the test MSE for each algorithm during training. B2B achieves better convergence speed and performance relative to other approaches. }
        \label{fig:elastic}
    \end{minipage}%
    \hfill
    \begin{minipage}[t]{.48\linewidth}
        % \centering
        \includegraphics[width=1.0\linewidth]{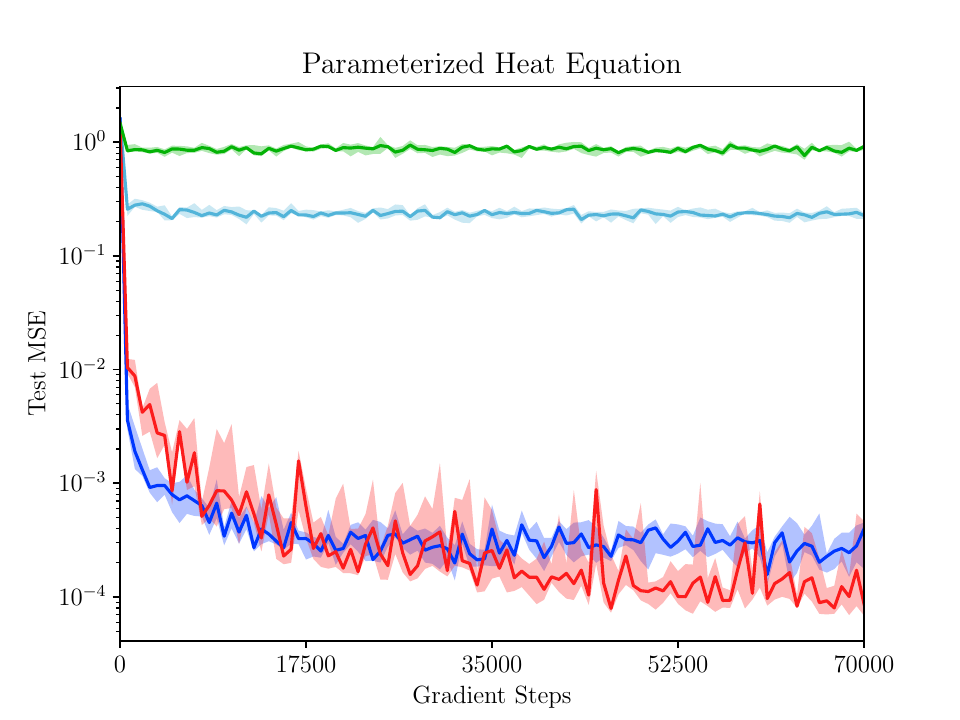}
        \caption{\textit{Training curves on the parameterized heat equation dataset.} This figure plots the test MSE for each algorithm during training. DeepONet achieves the best performance on this dataset, although it also demonstrates some instability with the occasional drop in performance to be on par with B2B. }
        \label{fig:heat}
    \end{minipage}%
\end{figure}
\begin{figure}[!h]
    \centering
    \subfloat[Basis-to-Basis]{
    \includegraphics[width=0.9\linewidth]{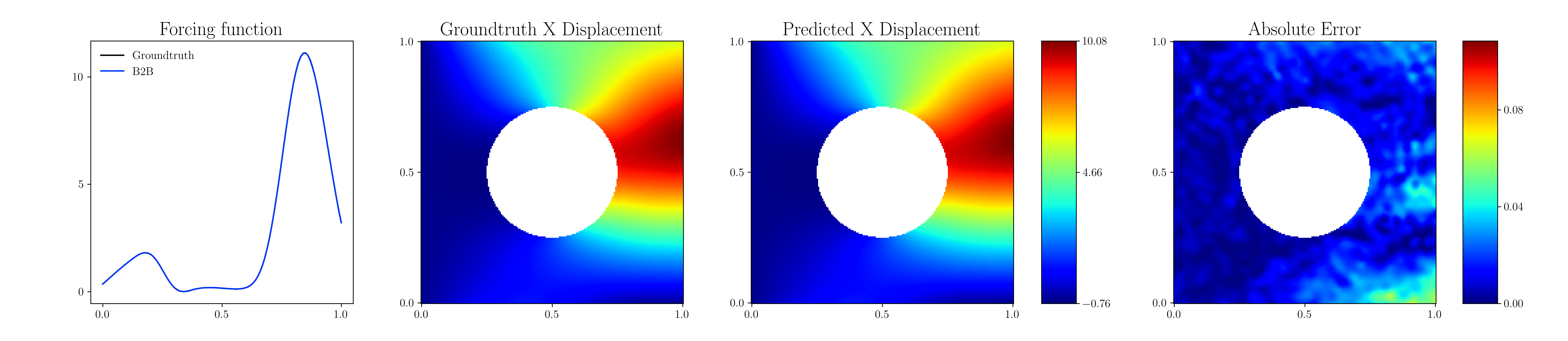}}
    \vspace{1em} 
    \subfloat[DeepONet]{
    \includegraphics[width=0.9\linewidth]{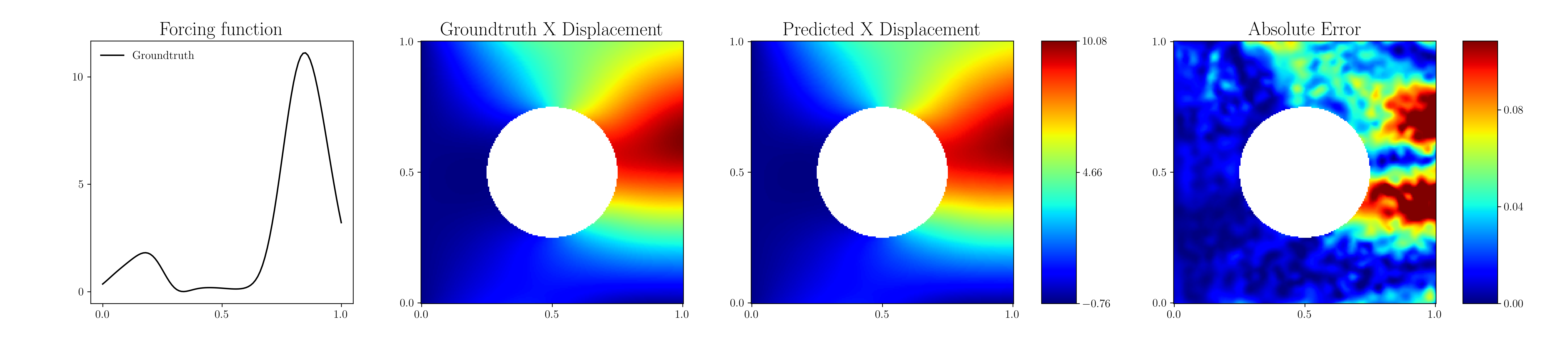}}
    \caption{\textit{Worst-case sample for the elastic plate dataset.} This figure compares B2B (top) and DeepONet (bottom) on the worst-case sample for  B2B from the elastic plate dataset. Left: the example input function, a forcing function, applied to the side of the plate (black) and B2B's estimate (blue). The B2B estimate is calculated using partial data from the ground truth forcing function and the input basis functions. The low-loss estimate suggests that the basis function coefficients effectively capture the forcing function. Center-left: Ground truth normalized x-direction displacement of the panel. Center-right: Estimated x-direction displacement using the corresponding algorithm. Right: Absolute error between the estimated and ground truth x-direction displacements. The color bar range represents $1\%$ of the absolute difference shown in the middle panes.}
    \label{fig:elasticity_error}
\end{figure}

The test error convergence plots are shown in Figure~\ref{fig:heat}.  B2B and DeepONet, while similar in architecture for this dataset, differ primarily in their approach to training the output basis. DeepONet employs end-to-end training, allowing the network to optimize the basis representation directly from data. In contrast, B2B uses a function encoder algorithm to train its output basis, aiming to reproduce the output function space accurately. Despite this distinction, both methods utilize similar network components (branch network for DeepONet, operator network for B2B) to map inputs to basis coefficients. 
The convergence plot depicts a slightly better performance of DeepONet over B2B. This observation suggests that DeepONet's end-to-end training may be more effective when there is no explicit input function space. 
This implies that B2B's strengths are most evident in scenarios with a well-defined input function space, where it can leverage this structure to enhance generalization and interpretability. A representative sample depicting the predictions of DeepONet and B2B is shown in Figure~\ref{fig:heat_accuracy}.

\begin{figure}[!h]
    \centering
    \subfloat[Basis-to-Basis]{
    \includegraphics[width=\linewidth]{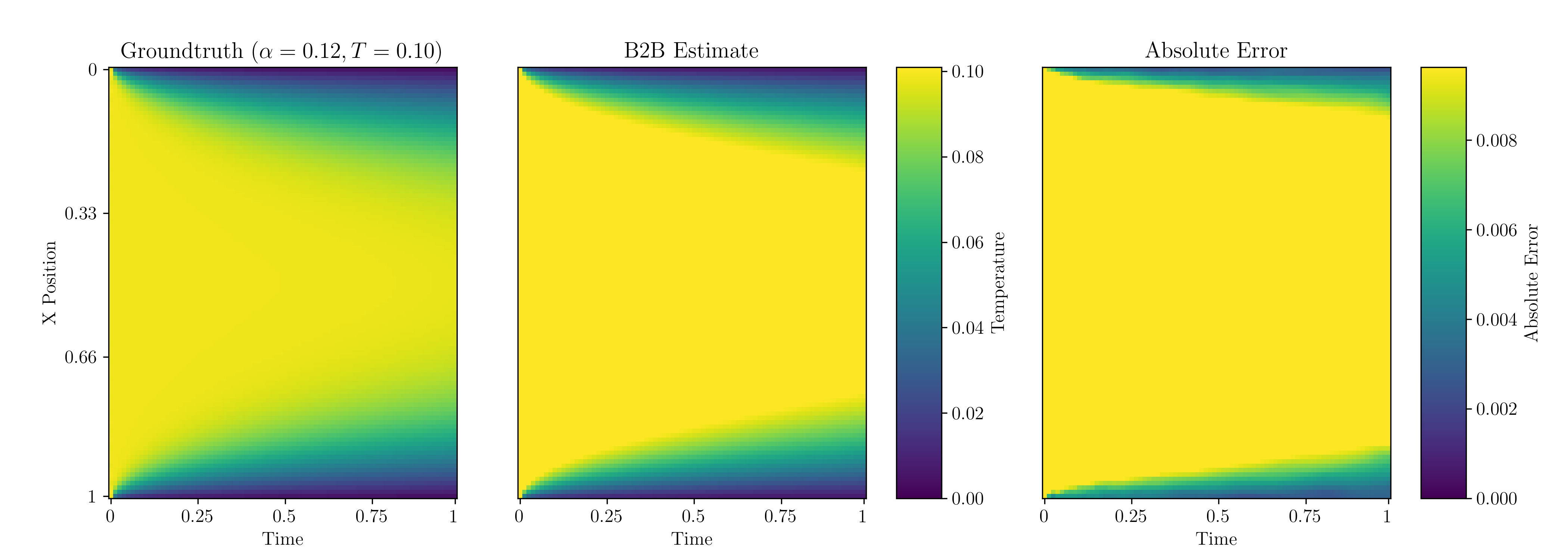}}
    \vspace{1em} 
    \subfloat[DeepONet]{
    \includegraphics[width=\linewidth]{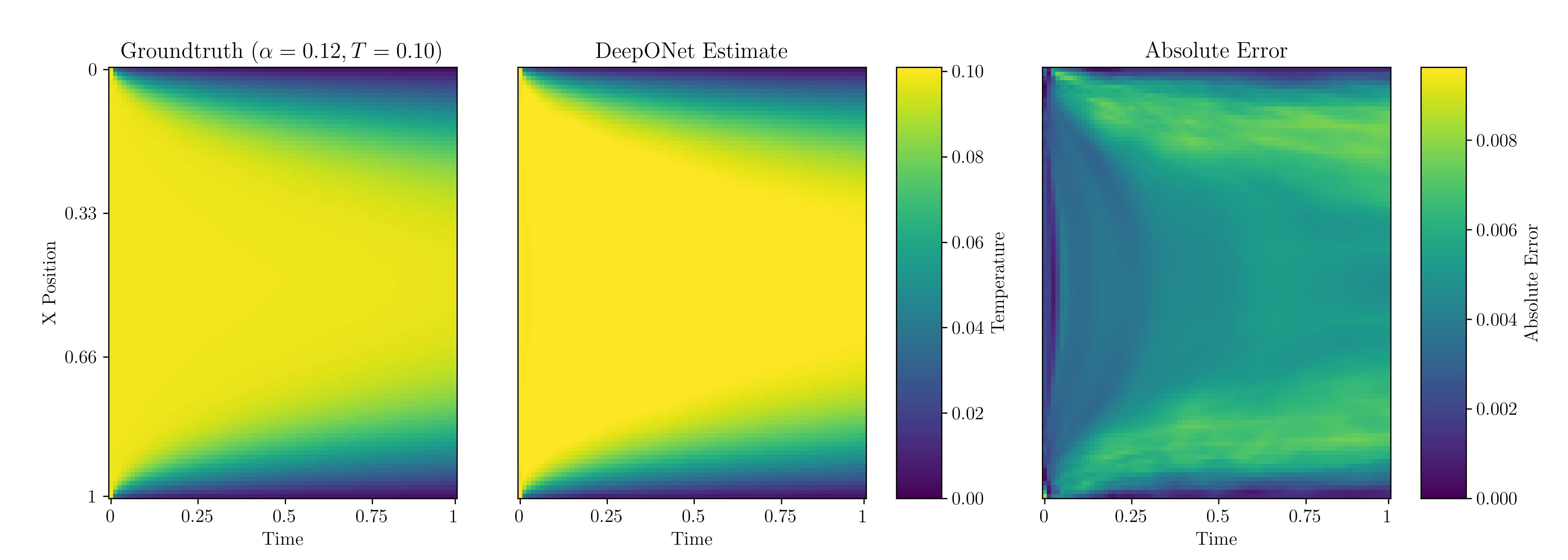}}
    \caption{\textit{Worst-case sample for the parameterized heat equation dataset.} This figure compares B2B (top) and DeepONet (bottom) on the worst-case sample for  B2B on the parameterized heat equation dataset.  Left: The ground truth relationship between the x position, time, and temperature. Center: the approximated relationship between x position, time, and temperature. Right: The absolute error in the approximation. The color bar ranges from 0 to 10\% of the maximum difference in output value, with bright yellow indicating $>\!10\%$ error.}
    \label{fig:heat_accuracy}
\end{figure}

\subsection{Burger's Equation}
\noindent For the final example, we consider the 1D wave propagation problem described by Burger's equation as:
\begin{equation}\label{burgers_eqn_IBC}
    \begin{split}
&\frac{\partial u}{\partial t} (x,t)=  \nu \frac{\partial^2 u}{\partial x^2} (x,t)- u\frac{\partial u}{\partial x} (x,t) \text{ on }  \ \Omega: (x,t) \in [0,1]^2 , \\
\text{such that:  } &u(x,0)= f(x), \; u(0,t)= u(1,t)\; \text{and } \frac{\partial u}{\partial x}(0,t) =  \frac{\partial u}{\partial x}(1,t),
  \end{split}
\end{equation}
where $u(x,t)$ denotes the evolving spatio-temporal velocity field, defined over the spatial and time coordinates, $(x,t)$, with viscosity parameter $\nu$. Our objective is to learn the transformation $\mathcal{T}$ that maps the initial velocity field $u(x,0)$ to the velocity field $u(x,t)$ at any subsequent time $t$ across the entire spatial domain, i.e., $\mathcal{T}: u{x,0}\rightarrow u(x,t)$. To generate varying initial conditions, we sample from a Gaussian Random Field (GRF) with mean zero and covariance operator $25^2(-\Delta + 5^2I)^{-4}$, subject to periodic boundary conditions. The model training utilizes random realizations of $f(x)$ drawn from this GRF.

Figure \ref{fig:burger_quantitative} presents our comparative analysis. The B2B architecture demonstrates superior performance in both convergence rate and asymptotic accuracy compared to DeepONet. Notably, both 2-state and POD DeepONets fail to converge on this challenging dataset. The significant performance gap between SVD (a linear operator) and the neural architectures (B2B and DeepONet) indicates the inherently nonlinear nature of the underlying operator. Figure \ref{fig:burger_qual} presents a detailed comparison between B2B and DeepONet on B2B's worst-performing test case.

\begin{figure}[!h]
    \centering
    \includegraphics[width=0.6\linewidth]{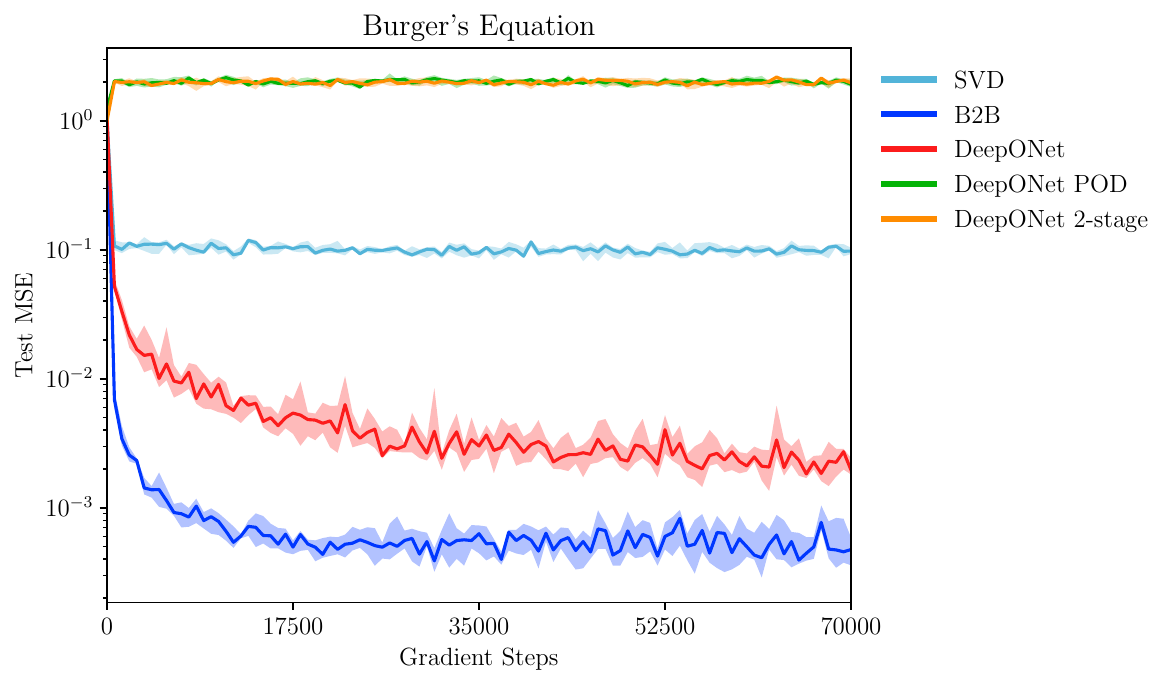}
    \caption{\textit{Training curves on the Burger's equation dataset.} This figure plots the test MSE for each algorithm during training. B2B achieves the best convergence speed and asymptotic performance on this dataset. Interestingly, DeepONet POD and 2-stage fail on this dataset. }
    \label{fig:burger_quantitative}
\end{figure}
\begin{figure}[!h]
    \centering
    \subfloat[Basis-to-Basis]{
    \includegraphics[width=0.9\linewidth]{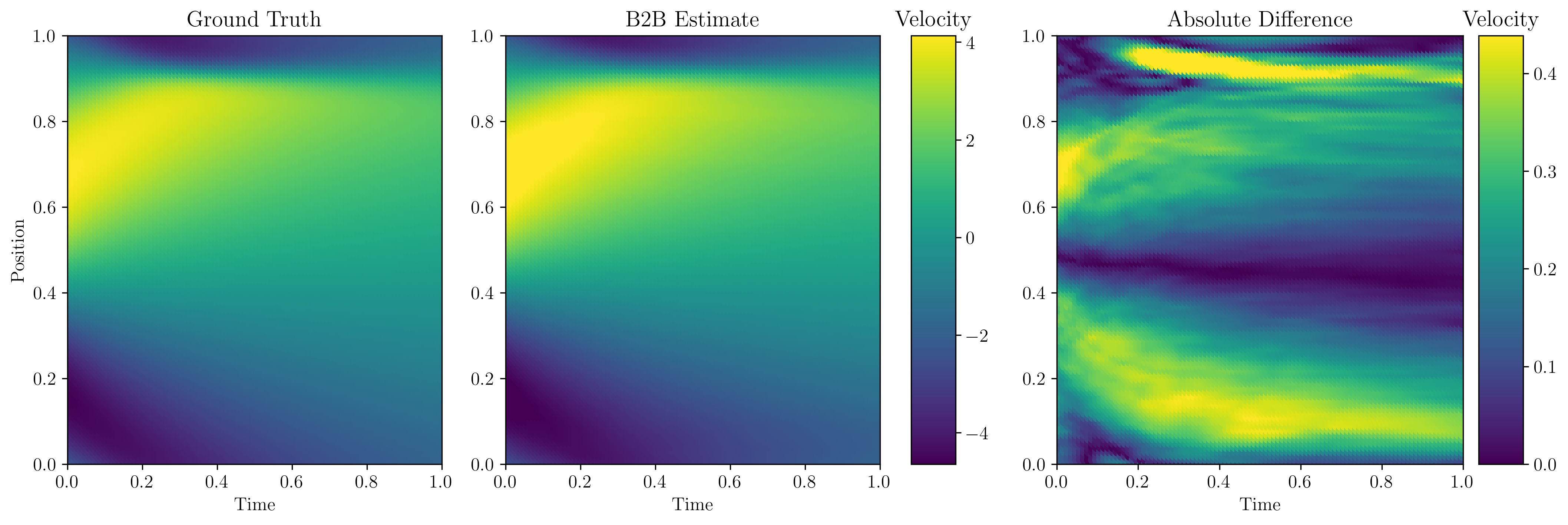}}
    \vspace{1em} 
    \subfloat[DeepONet]{
    \includegraphics[width=0.9\linewidth]{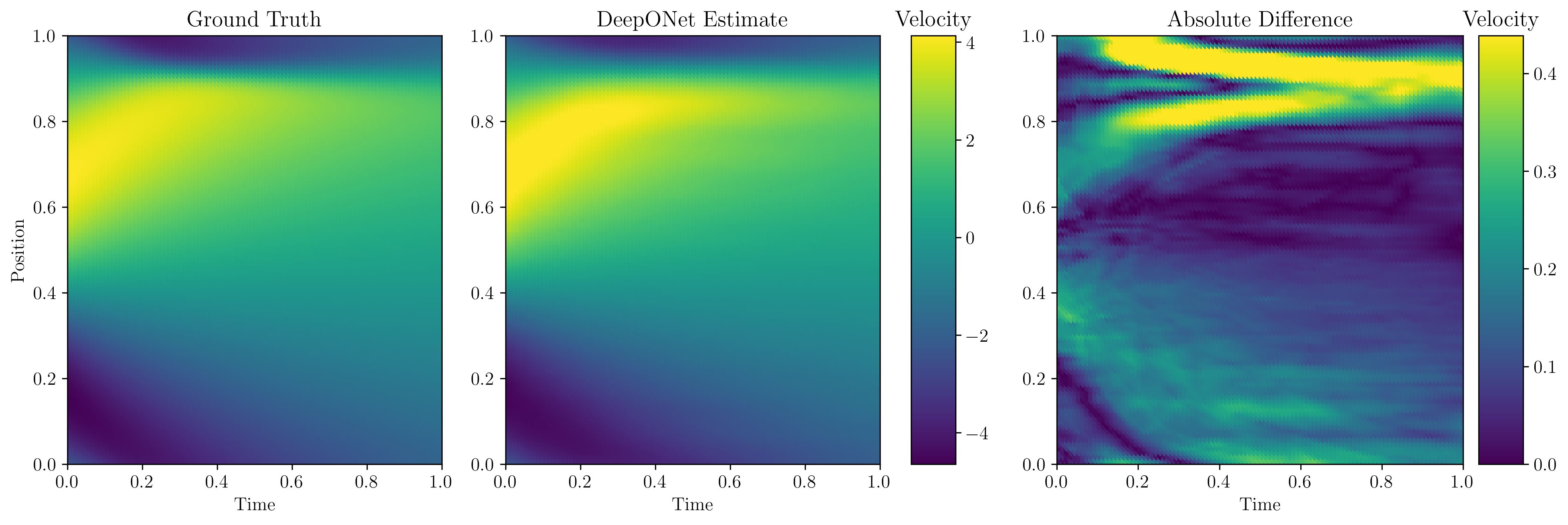}}
    \caption{\textit{Worst-case sample for the Burger's equation dataset.} This figure compares B2B (top) and DeepONet (bottom) on the worst-case sample for  B2B on Burger's equation dataset.  Left: The ground truth relationship between position, time, and fluid velocity. Center: the approximated relationship between position, time, and fluid velocity. Right: The absolute error in the approximation. The color bar ranges from 0 to 5\% of the maximum difference in output value, with bright yellow indicating $>\!5\%$ error.}
    \label{fig:burger_qual}
\end{figure}

\section{Ablations}
\label{sec:ablation}

Sensitivity to hyper-parameters is an important consideration for any algorithm. We evaluate the sensitivity of our approach to its two main hyper-parameters, the number of basis functions and the number of input sensors. Ablation figures indicate the Test MSE after 70,000 gradient steps, averaged over the last 10 steps to reduce noise. Lines and shaded areas indicate median, first, and third quartiles over 3 seeds for each value of hyper-parameter.  

\subsection{Number of Basis Functions}
A key hyper-parameter for our approach is the number of basis functions used for both input and output function spaces. By default, all experiments use $100$ basis functions, which is possible due to parameter sharing. We perform an ablation on several basis functions, both for our approach and for DeepONet. See Figure~\ref{fig:ablation_basis}. 

On the anti-derivative example, all approaches are insensitive to the number of basis functions. This is likely because the dimensionality of the function space is only three, and so any more than three basis functions are unnecessary. 
For the 2D Darcy flow dataset, all approaches see decreasing performance as the number of basis functions approaches 0. This implies the dimensionality of the function spaces is much higher. Interestingly, B2B degrades faster than DeepONet and its variants. This is likely because DeepONet uses an end-to-end loss function, and therefore learns the basis that achieves the lowest error on this end-to-end objective. In contrast, B2B learns basic functions to reproduce the individual function spaces. 
If the bases are over-specified, then they will fully represent each function space, and the operator can then be learned without issue. 
However, if too few basis functions are used, it may be the case that the basis that best reproduces the space does not coincide with the basis that performs best on the end-to-end operator loss. 
However, due to the scaling properties of function encoders, this issue may be avoided by simply choosing a large number of basis functions, e.g., 100, as this is enough for most problems. 

\begin{figure}[!h]
    \begin{minipage}{.48\linewidth}
        % \centering
        \includegraphics[width=1.0\linewidth]{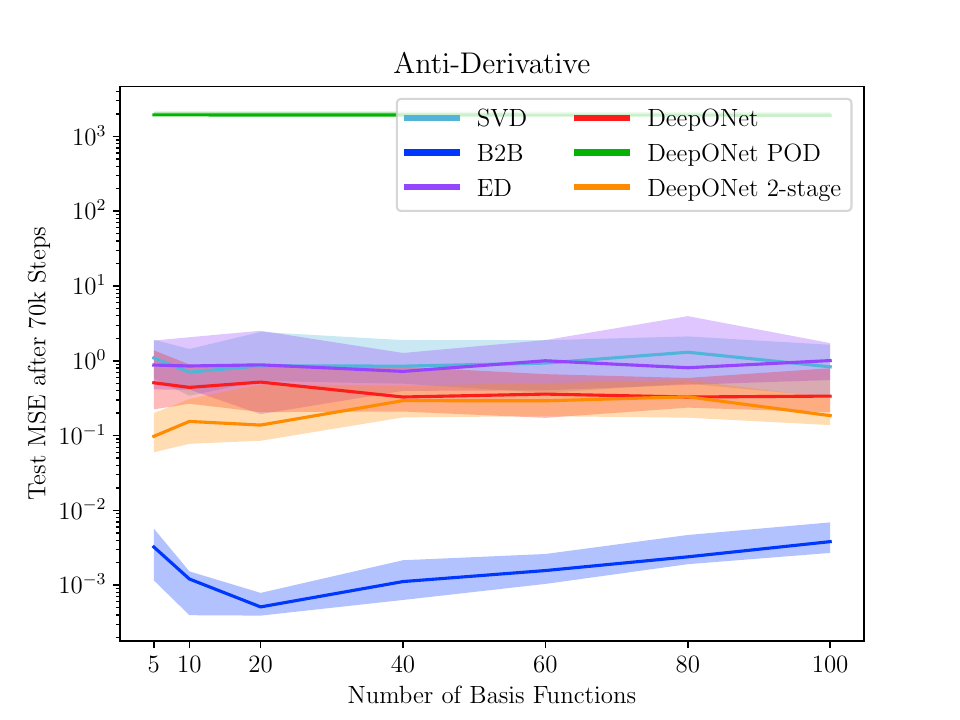}
    \end{minipage}%
    \hfill
    \begin{minipage}{.48\linewidth}
        % \centering
        \includegraphics[width=1.0\linewidth]{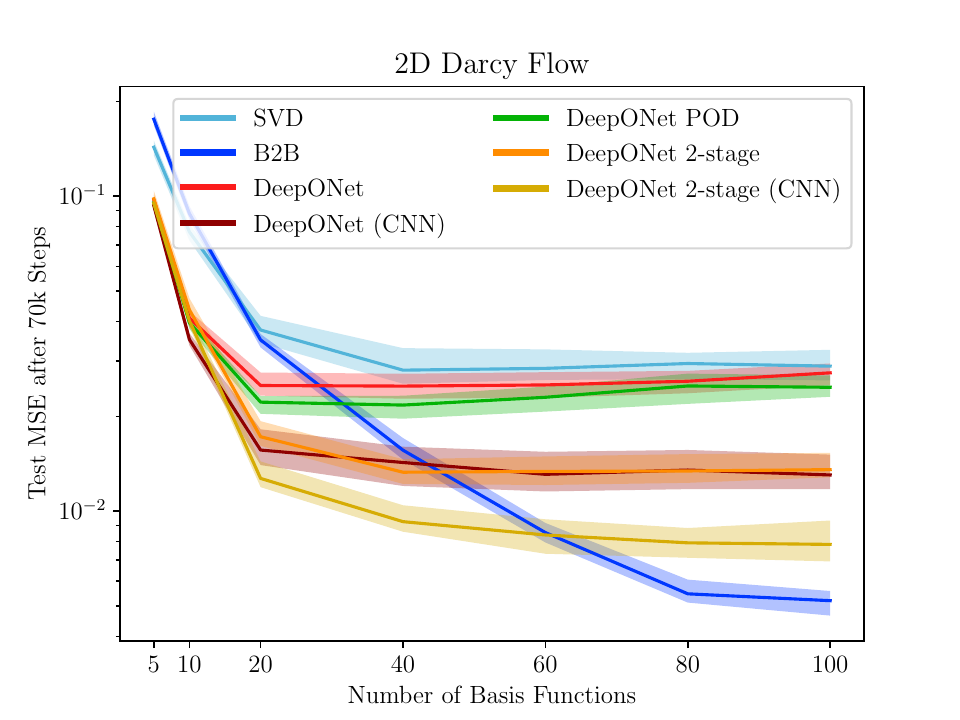}
    \end{minipage}%
    \caption{\textit{Ablation on how the number of basis functions affects performance}. Left: All algorithms are insensitive to the number of basis functions for the anti-derivative dataset because the dimensionality of this problem is low. Right: In contrast, the dimensionality of the 2D Darcy flow dataset is relatively high, as all approaches degrade as the number of basis functions approaches 0. B2B is more sensitive than other approaches as it does not use an end-to-end loss.  }
    \label{fig:ablation_basis}
\end{figure}
\begin{figure}[!h]
    \begin{minipage}{.48\linewidth}
        % \centering
        \includegraphics[width=1.0\linewidth]{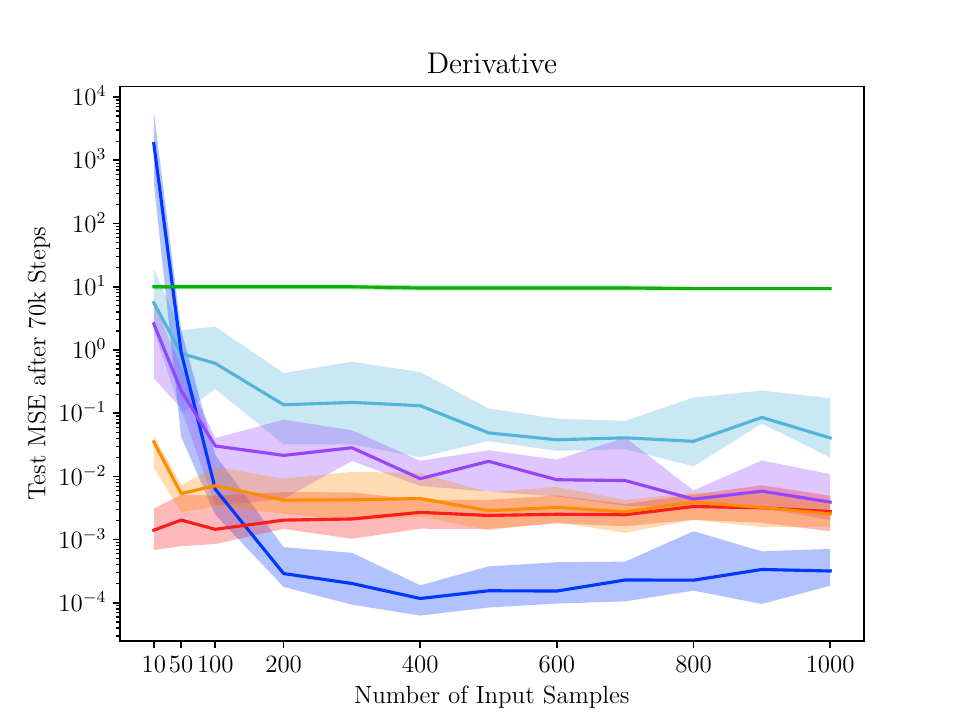}
    \end{minipage}%
    \hfill
    \begin{minipage}{.48\linewidth}
        % \centering
        \includegraphics[width=1.0\linewidth]{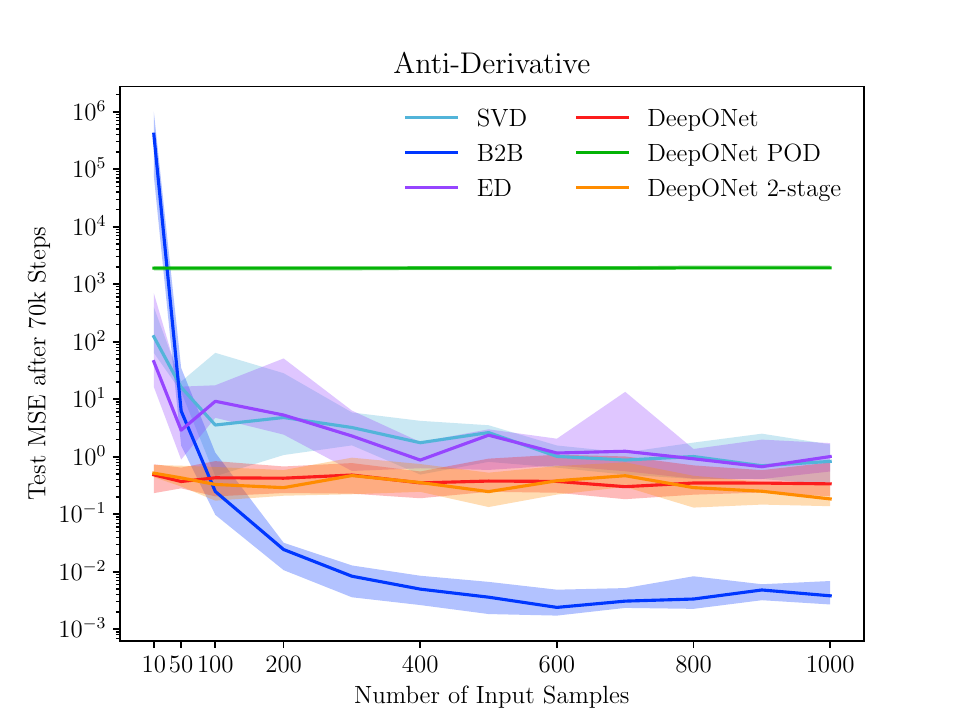}
    \end{minipage}%
    \caption{\textit{An Ablation on how the number of input samples affects performance.} Left: DeepONet is relatively insensitive to the number of input samples. B2B and its variants show a drop off in performance as the number of samples decreases to 0.  Right: The results are largely the same on the anti-derivative dataset.  }
    \label{fig:ablation_sensor}
\end{figure}
\begin{figure}[!h]
    \centering
    \includegraphics[width=0.68\linewidth]{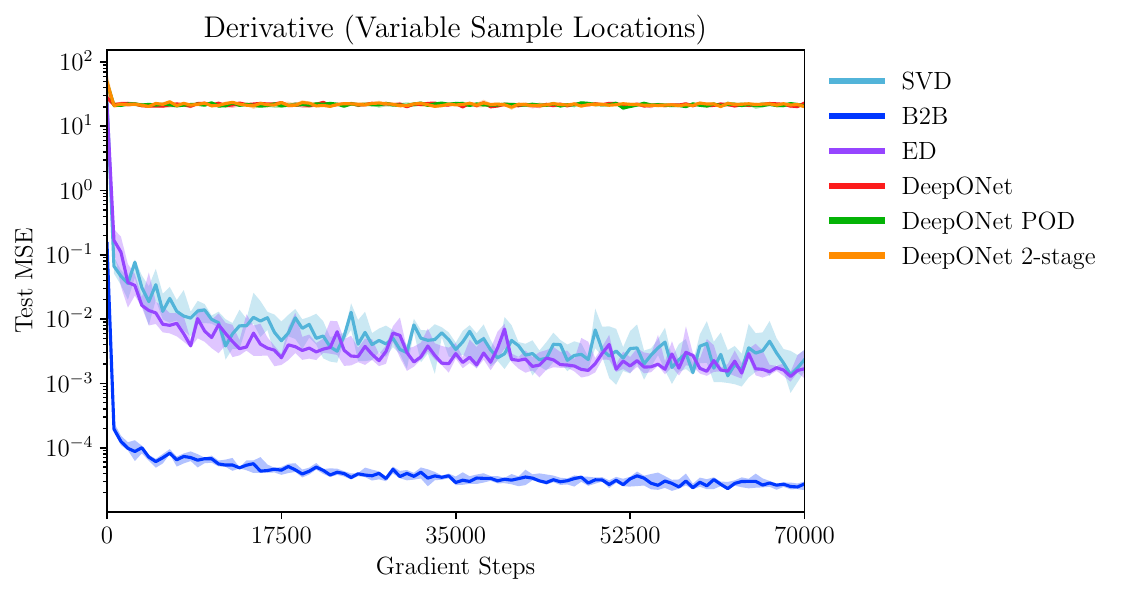}
    \caption{\textit{An ablation with variable sample locations.} DeepONet and its variants fail to converge if the input samples have varied locations for every function, as expected. In contrast, B2B and its variants are unaffected because they use least squares to compute basis function coefficients, a method which is insensitive to the sample locations.}
    \label{fig:varying-sensor-loc}
\end{figure}
\subsection{Number of Input Samples}
The number of input sensors used to identify the input function is another important hyper-parameter. An ideal algorithm would be able to make use of large amounts of data while still performing reasonably well in small data settings. We perform ablations on the derivative and anti-derivative examples. We observe that DeepONet and its variants are relatively insensitive to the number of sensors, though its performance seems to improve on Derivative with \textit{less} sensors. This suggests it is not making good use of the data in large data settings. B2B and its variants are also relatively insensitive until a critical threshold is reached at low data settings. This is likely due to the least-squares solution over fitting if the number of samples is less than the number of basis functions. In future work, we plan to investigate how to best correct this by regularizing the least-squares solution.

\subsection{Variable Input Sample Locations}
Many neural operator architectures like DeepONet are designed with fixed input dimensions. This design choice stems from the structure of their neural networks, particularly the branch network, which expects a consistent number of input features across all samples. Mathematically, suppose we denote the input function as $f(x)$. In that case, these methods typically require evaluations at fixed locations ${x_1, \ldots, x_m}$, such that the input to the network is the fixed vector $(f(x_1), \ldots, f(x_m)) \in \mathbb{R}^m$.

This rigid structure presents significant challenges when dealing with datasets collected at variable locations. In practical scenarios, sensor placements may vary due to experimental constraints, equipment limitations, or the nature of the phenomenon being studied. Consequently, existing methods struggle to generalize in these settings, as they cannot naturally accommodate inputs of varying dimensionality or structure.

B2B offers a solution to this problem by leveraging the least-squares method. 
Instead of directly mapping input data to the coefficients of the basis functions as in DeepONet, B2B uses least squares to compute a function representation and then maps this representation to a corresponding representation in the output space. 
As least squares is well-defined for any set of input points regardless of their locations, B2B can easily handle a variable number of input samples, or input samples at variable locations.
In other words, this approach decouples the operator learning problem from the specific sampling locations of the input data.

\Cref{fig:varying-sensor-loc} illustrates the MSE test performance of B2B and DeepONet variants on learning the derivative operator with variable input sensor locations. The results demonstrate a stark contrast between the two approaches. The B2B approach (both SVD and standard variants) shows robust performance, achieving and maintaining low test MSE (around $10^{-3}$ to $10^{-4}$) throughout the training process. B2B successfully learns the derivative operator despite the variability in input sampling locations. In contrast, all variants of DeepONet (vanilla, POD, and two-stage) exhibit significantly higher test MSE, with errors consistently above $10^1$.  By learning basis functions that span the entire input domain, B2B can effectively interpolate between arbitrary sampling points, enabling it to generalize well to inputs with varying sensor locations.

\section{Conclusion \& Future Work}
\label{sec:conclusion}
This paper introduces a novel framework for learning nonlinear operators on Hilbert spaces using the theory of function encoders. Our approach leverages the structure of Hilbert spaces to learn efficient and interpretable representations of function spaces and the operators acting on them. We developed three operator learning variants using our framework: Basis-to-Basis operator learning (B2B), Singular Value Decomposition (SVD), and Eigendecomposition (ED). The B2B variant offers particular flexibility in handling linear and nonlinear operators while maintaining a clear separation between function space representation and operator mapping.

We demonstrated the effectiveness of our approach on several benchmark problems, including derivative and anti-derivative operators, 1D and 2D Darcy flow equations, linear elasticity, Burger's equation, and a parameterized heat equation. Our results showed that the function encoder-based methods, especially the B2B variant, consistently outperformed the popular DeepONet architecture and its variants in accuracy and convergence speed. The ability to generalize well across different function spaces and handle variable input sensor locations without retraining highlights its robustness and practical applicability.

We performed linear operator analysis to provide deeper insights into the learned operators, particularly for the SVD and ED variants. While these variants are restricted to linear operators, they offer robust and interpretable representations, enabling analysis of properties such as the operator's spectrum, and contributing to a better understanding of the underlying physical phenomena. Our approach shows remarkable adaptability across various problem domains, from simple linear operators to complex PDEs, without requiring modifications to the underlying algorithm.

In future work, we intend to explore adaptive basis function learning, where the number and form of basis functions can be dynamically adjusted based on the complexity of the function spaces involved. Furthermore, as shown in \cite{ingebrand2024zeroshottransferneuralodes}, function encoders are architecture-agnostic and can therefore use more exotic architectures depending on the problem, e.g.,\ neural ODEs for operators in dynamical systems or GNNs for operators acting on graphs. Lastly, function encoders apply to any Hilbert space, including probability distributions. Therefore, this work opens the door to modeling stochastic operators. 

\section*{Acknowledgments}
The authors would like to acknowledge Dr.\ Bahador Bahmani at Johns Hopkins University for generating the initial 1D nonlinear Darcy dataset, which was used for preliminary experiments. The dataset has since been replaced. 

This work was funded in part by AFOSR FA9550-19-1-0005, NSF 2214939, NSF 2339678, NSF 2438193, and HR0011-24-9-0431.  Any opinions, findings, conclusions, or recommendations expressed in this material are those of the author(s) and do not necessarily reflect the views of the funding organizations.
% National Science Foundation.

%%%%%%%%%%%%%%%%%%%%%%%%%
%%%%%%%%%%%%%%%%%%%%%%%%%
\bibliographystyle{elsarticle-num}
\bibliography{bibliography}

\begin{thebibliography}{10}
\expandafter\ifx\csname url\endcsname\relax
  \def\url#1{\texttt{#1}}\fi
\expandafter\ifx\csname urlprefix\endcsname\relax\def\urlprefix{URL }\fi
\expandafter\ifx\csname href\endcsname\relax
  \def\href#1#2{#2} \def\path#1{#1}\fi

\bibitem{lu2021learning}
L.~Lu, P.~Jin, G.~Pang, Z.~Zhang, G.~E. Karniadakis, Learning nonlinear operators via {DeepONet} based on the universal approximation theorem of operators, Nature Machine Intelligence 3~(3) (2021) 218--229.

\bibitem{chen1995universal}
T.~Chen, H.~Chen, Universal approximation to nonlinear operators by neural networks with arbitrary activation functions and its application to dynamical systems, Transactions on Neural Networks 6~(4) (1995) 911--917.

\bibitem{li2021fourier}
Z.~Li, N.~B. Kovachki, K.~Azizzadenesheli, B.~liu, K.~Bhattacharya, A.~Stuart, A.~Anandkumar, {Fourier} neural operator for parametric partial differential equations, in: International Conference on Learning Representations, 2021.

\bibitem{geelen2023operator}
R.~Geelen, S.~Wright, K.~Willcox, Operator inference for non-intrusive model reduction with quadratic manifolds, Computer Methods in Applied Mechanics and Engineering 403 (2023) 115717.

\bibitem{kontolati2023influence}
K.~Kontolati, S.~Goswami, M.~D. Shields, G.~E. Karniadakis, On the influence of over-parameterization in manifold based surrogates and deep neural operators, Journal of Computational Physics 479 (2023) 112008.

\bibitem{ingebrand2024zeroshottransferneuralodes}
T.~Ingebrand, A.~J. Thorpe, U.~Topcu, Zero-shot transfer of neural {ODEs}, in: Advances in Neural Information Processing Systems, Vol.~37, Curran Associates, Inc., 2024.

\bibitem{ingebrand2024zeroshotreinforcementlearningfunction}
T.~Ingebrand, A.~Zhang, U.~Topcu, Zero-shot reinforcement learning via function encoders, in: International Conference on Machine Learning, Vol. 235 of Proceedings of Machine Learning Research, PMLR, 2024, pp. 21007--21019.

\bibitem{transformerOperator}
Z.~Hao, Z.~Wang, H.~Su, C.~Ying, Y.~Dong, S.~Liu, Z.~Cheng, J.~Song, J.~Zhu, {GNOT:} {A} general neural operator transformer for operator learning, in: International Conference on Machine Learning, Vol. 202 of Proceedings of Machine Learning Research, {PMLR}, 2023, pp. 12556--12569.

\bibitem{HE2024107258}
J.~He, S.~Kushwaha, J.~Park, S.~Koric, D.~Abueidda, I.~Jasiuk, Sequential deep operator networks {(S-DeepONet)} for predicting full-field solutions under time-dependent loads, Engineering Applications of Artificial Intelligence 127 (2024) 107258.

\bibitem{KUSHWAHA2024104266}
S.~Kushwaha, J.~Park, S.~Koric, J.~He, I.~Jasiuk, D.~Abueidda, Advanced deep operator networks to predict multiphysics solution fields in materials processing and additive manufacturing, Additive Manufacturing 88 (2024) 104266.

\bibitem{recurrent}
K.~Michalowska, S.~Goswami, G.~E. Karniadakis, S.~Riemer{-}S{\o}rensen, Neural operator learning for long-time integration in dynamical systems with recurrent neural networks, in: International Joint Conference on Neural Networks, {IEEE}, 2024, pp. 1--8.

\bibitem{riemann}
A.~Peyvan, V.~Oommen, A.~D. Jagtap, G.~E. Karniadakis, {RiemannONets}: Interpretable neural operators for riemann problems, Computer Methods in Applied Mechanics and Engineering 426 (2024) 116996.

\bibitem{srno}
M.~Wei, X.~Zhang, Super-resolution neural operator, in: Conference on Computer Vision and Pattern Recognition, {IEEE}, 2023, pp. 18247--18256.

\bibitem{ZHANG2024112638}
J.~Zhang, S.~Zhang, J.~Shen, G.~Lin, Energy-dissipative evolutionary deep operator neural networks, Journal of Computational Physics 498 (2024) 112638.

\bibitem{d2no}
Z.~Zhang, C.~Moya, L.~Lu, G.~Lin, H.~Schaeffer, {D2NO:} efficient handling of heterogeneous input function spaces with distributed deep neural operators, Computer Methods in Applied Mechanics and Engineering 428 (2024) 117084.

\bibitem{tripura2023wavelet}
T.~Tripura, S.~Chakraborty, Wavelet neural operator for solving parametric partial differential equations in computational mechanics problems, Computer Methods in Applied Mechanics and Engineering 404 (2023) 115783.

\bibitem{cao2024laplace}
Q.~Cao, S.~Goswami, G.~E. Karniadakis, {Laplace} neural operator for solving differential equations, Nature Machine Intelligence 6~(6) (2024) 631--640.

\bibitem{pideepo}
S.~Goswami, A.~Bora, Y.~Yu, G.~E. Karniadakis, Physics-Informed Deep Neural Operator Networks, Springer International Publishing, 2023, pp. 219--254.

\bibitem{pideeponetstudy}
K.~Kobayashi, J.~Daniell, S.~B. Alam, Improved generalization with deep neural operators for engineering systems: Path towards digital twin, Engineering Applications of Artificial Intelligence 131 (2024) 107844.

\bibitem{KORIC2023123809}
S.~Koric, D.~W. Abueidda, Data-driven and physics-informed deep learning operators for solution of heat conduction equation with parametric heat source, International Journal of Heat and Mass Transfer 203 (2023) 123809.

\bibitem{li2023phasefielddeeponetphysicsinformeddeep}
W.~Li, M.~Z. Bazant, J.~Zhu, Phase-field {DeepONet}: Physics-informed deep operator neural network for fast simulations of pattern formation governed by gradient flows of free-energy functionals (2023).
\newblock \href {http://arxiv.org/abs/2302.13368} {\path{arXiv:2302.13368}}.

\bibitem{reno}
F.~Bartolucci, E.~de~B{\'{e}}zenac, B.~Raonic, R.~Molinaro, S.~Mishra, R.~Alaifari, Representation equivalent neural operators: a framework for alias-free operator learning, in: Advances in Neural Information Processing Systems, 2023.

\bibitem{rompca}
C.~Audouze, F.~De~Vuyst, P.~B. Nair, Reduced-order modeling of parameterized {PDEs} using time-space-parameter principal component analysis, International Journal for Numerical Methods in Engineering 80~(8) (2009) 1025--1057.

\bibitem{BennerGW15}
P.~Benner, S.~Gugercin, K.~Willcox, A survey of projection-based model reduction methods for parametric dynamical systems, {SIAM} Review 57~(4) (2015) 483--531.

\bibitem{SWISCHUK2019704}
R.~Swischuk, L.~Mainini, B.~Peherstorfer, K.~Willcox, Projection-based model reduction: Formulations for physics-based machine learning, Computers \& Fluids 179 (2019) 704--717.

\bibitem{gappypod}
R.~Everson, L.~Sirovich, {Karhunen--Lo\`{e}ve} procedure for gappy data, Journal of the Optical Society of America A 12~(8) (1995) 1657--1664.

\bibitem{reducedbasis}
J.~S. Hesthaven, G.~Rozza, B.~Stamm, Reduced Basis Methods, Springer International Publishing, 2016, pp. 27--43.

\bibitem{podnn}
J.~S. Hesthaven, S.~Ubbiali, Non-intrusive reduced order modeling of nonlinear problems using neural networks, Journal of Computational Physics 363 (2018) 55--78.

\bibitem{podapplication}
P.~Holmes, J.~L. Lumley, G.~Berkooz, Turbulence, Coherent Structures, Dynamical Systems and Symmetry, Cambridge Monographs on Mechanics, Cambridge University Press, 1996.

\bibitem{nomad}
J.~H. Seidman, G.~Kissas, P.~Perdikaris, G.~J. Pappas, {NOMAD:} nonlinear manifold decoders for operator learning, in: Advances in Neural Information Processing Systems, 2022.

\bibitem{coral}
L.~Serrano, L.~L. Boudec, A.~K. Koupa{\"{\i}}, T.~X. Wang, Y.~Yin, J.~Vittaut, P.~Gallinari, Operator learning with neural fields: Tackling {PDEs} on general geometries, in: Advances in Neural Information Processing Systems, 2023.

\bibitem{pca_nn}
K.~Bhattacharya, B.~Hosseini, N.~B. Kovachki, A.~M. Stuart, Model reduction and neural networks for parametric {PDEs}, The {SMAI} Journal of computational mathematics 7 (2021) 121--157.

\bibitem{He_2024}
J.~He, S.~Koric, D.~Abueidda, A.~Najafi, I.~Jasiuk, {Geom-DeepONet}: A point-cloud-based deep operator network for field predictions on {3D} parameterized geometries, Computer Methods in Applied Mechanics and Engineering 429 (2024) 117130.

\bibitem{dimon}
M.~Yin, N.~Charon, R.~Brody, L.~Lu, N.~Trayanova, M.~Maggioni, {DIMON}: Learning solution operators of partial differential equations on a diffeomorphic family of domains (2024).
\newblock \href {http://arxiv.org/abs/2402.07250} {\path{arXiv:2402.07250}}.

\bibitem{bahmani2024resolution}
B.~Bahmani, S.~Goswami, I.~G. Kevrekidis, M.~D. Shields, A resolution independent neural operator (2024).
\newblock \href {http://arxiv.org/abs/2407.13010} {\path{arXiv:2407.13010}}.

\bibitem{lu2022comprehensive}
L.~Lu, X.~Meng, S.~Cai, Z.~Mao, S.~Goswami, Z.~Zhang, G.~E. Karniadakis, A comprehensive and fair comparison of two neural operators (with practical extensions) based on fair data, Computer Methods in Applied Mechanics and Engineering 393 (2022) 114778.

\bibitem{lee2024training}
S.~Lee, Y.~Shin, On the training and generalization of deep operator networks, Journal on Scientific Computing 46~(4) (2024) C273--C296.

\bibitem{li2018visualizinglosslandscapeneural}
H.~Li, Z.~Xu, G.~Taylor, C.~Studer, T.~Goldstein, Visualizing the loss landscape of neural nets, in: Advances in Neural Information Processing Systems, 2018.

\bibitem{kahana2023geometry}
A.~Kahana, E.~Zhang, S.~Goswami, G.~Karniadakis, R.~Ranade, J.~Pathak, On the geometry transferability of the hybrid iterative numerical solver for differential equations, Computational Mechanics 72~(3) (2023) 471--484.

\bibitem{goswami2022deep}
S.~Goswami, K.~Kontolati, M.~D. Shields, G.~E. Karniadakis, Deep transfer operator learning for partial differential equations under conditional shift, Nature Machine Intelligence 4~(12) (2022) 1155--1164.

\bibitem{mandl2024separable}
L.~Mandl, S.~Goswami, L.~Lambers, T.~Ricken, Separable {DeepONet}: Breaking the curse of dimensionality in physics-informed machine learning (2024).
\newblock \href {http://arxiv.org/abs/2407.15887} {\path{arXiv:2407.15887}}.

\bibitem{anandkumar2020neural}
A.~Anandkumar, K.~Azizzadenesheli, K.~Bhattacharya, N.~Kovachki, Z.~Li, B.~Liu, A.~Stuart, Neural operator: Graph kernel network for partial differential equations, in: International Conference on Learning Representations, 2020.

\end{thebibliography}
\newpage
\section*{Appendix}
\appendix

\section{Dataset Details}
\begin{table}[h]
    \caption{The parameters for each dataset. $N$ indicates the number of functions that are trained on, e.g.\ $\mathcal{D} = D_1, D_2, ..., D_N$. Note that the derivative and anti-derivative datasets have $700{,}000$ functions because they are generative, there were $70{,}000$ gradient steps, and functions are sampled in batches of $10$. $m$ indicates the number of input points, i.e.\ $\{\big (x_i, f(x_i) \big )\}_{i=1}^m$. $p$ indicates the number of output points, i.e.\ $\{\big (y_i, \mathcal{T}f(y_i)\big ) \}_{i=1}^p$.}
    \label{tab:datasets}
    \centering
    \small
    \begin{tabular}{lcccccc} %|l|c|c|c|c|c|c|} 
        \toprule
        & $N_{train}$ & $N_{test}$ & $m$ & $p$ \\
        \midrule
        Anti-Derivative & $700{,}000$ & $700{,}000$ & $1{,}000$ & $10{,}000$ \\
        %\midrule
        Derivative & $700{,}000$ & $700{,}000$ &$1{,}000$ & $10{,}000$ \\
        %\midrule
        1D Darcy Flow & $800$ & $200$  & $40$ & $40$ \\
        %\midrule
        2D Darcy Flow & $51{,}000$ & $7{,}000$ & $961$ & $450$ \\
        %\midrule
        Elastic Plate & $1{,}850$ &  $100$ & $101$ & $1{,}048$ \\
        %\midrule
        Parameterized Heat Equation & $250$ & $30$ & $1$ & $784{,}080$ \\
        Burger's Equation & $2000$ & $500$ & $101$ & $10{,}201$ \\
        \bottomrule
    \end{tabular}
\end{table}

\section{Theoretical details of deep operator network (DeepONet)}

Neural operator learning, which involves employing DNNs to learn mappings between infinite-dimensional function spaces, has recently gained significant attention, particularly for its applications in learning ODEs and PDEs. A classical solution operator learning task then involves the learning of mapping across a range of scenarios, \textit{e.g.}, different domain geometries, input parameters, and initial and boundary conditions to the solution of the underlying ODE/PDE system. At the moment, there are a plethora of different neural operators, among which we can distinguish meta-architectures, \textit{e.g.}, deep operator networks (DeepONet)~\cite{lu2021learning} motivated by the universal approximation theorem for operators \cite{chen1995universal} and operators based on integral transforms, \textit{e.g.}, the Fourier neural operator (FNO) \cite{li2021fourier}, wavelet neural operator (WNO) \cite{tripura2023wavelet}, the graph kernel network (GKN) \cite{anandkumar2020neural} and the Laplace neural operator (LNO) \cite{cao2024laplace}, to name a few. 

In this work, we have compared the accuracy and robustness of the function encoder-based operator learning approaches with DeepONet. For all the applications, the DeepONet framework is employed to construct the solution operator, $\mathcal {T}$ that maps the input space, $\mathcal G$ to the output space, $\mathcal H$. Let $\Omega \subset \mathbb{R}^D$ be a bounded open set in the physical space, and let $\mathcal{G} = \mathcal{G}(\Omega; \mathbb{R}^{d_x})$ and $\mathcal{H} = \mathcal{H}(\Omega; \mathbb{R}^{d_y})$ be two separable functional spaces for which the mapping is to be learned. Furthermore, assume that $\mathcal {T}: \mathcal{G} \rightarrow \mathcal{H}$ is a nonlinear map arising from the solution of an unknown static or time-dependent PDE. The objective is to approximate the nonlinear operator via the following parametric mapping
\begin{equation}
\begin{aligned}
    \mathcal {T}: \mathcal{G} \times \mathbf{\Theta} \rightarrow \mathcal{H} \hspace{15pt} \text{or}, \hspace{15pt} \mathcal {T}_{\boldsymbol{\theta}}: \mathcal{G} \rightarrow \mathcal{H}, \hspace{5pt} \boldsymbol{\theta} \in \mathbf{\Theta},
\end{aligned}
\end{equation}
where $\mathbf{\Theta}$ is a finite-dimensional parameter space.
The optimal parameters $\boldsymbol{\theta}^*$ are learned via the training of a neural operator with backpropagation based on a labeled dataset obtained from the questionnaire algorithm. 

The DeepONet architecture consists of two concurrent DNNs, a branch network, and a trunk network. The branch network encodes the varying input realizations discretized at $m$ fixed sensor locations, $g(\boldsymbol{x}_i)_{i=1}^{m}$, where $g\in \mathcal G$, and the trunk network takes the spatial and temporal locations on the bounded domain $(\boldsymbol{x}_i, t_i)$, where the unknown PDE is evaluated. The solution operator for an input realization, $g_1$, can be expressed as:
\begin{equation}\label{eq:output_deeponets}
    \begin{aligned}
      \mathcal {T}_{\boldsymbol{\theta}}(g_1)(\boldsymbol{x}_j, t_j) = \sum_{i = 1}^{p}b_i(g_1)\cdot tr_i(\boldsymbol{x}_j, t_j),   
    \end{aligned}
\end{equation}
where $[b_1, b_2, \ldots, b_p]^T$ is the output vector of the branch net, $[tr_1, tr_2, \ldots, tr_p]^T$ the output vector of the trunk net and $p$ denotes a hyperparameter that controls the size of the output layer of both the branch and trunk networks. The trainable parameters of the DeepONet, represented by $\boldsymbol{\theta}$ in Equation~\eqref{eq:output_deeponets}, are obtained by minimizing a data-driven loss function, which is expressed as:
\begin{equation}
\begin{aligned}
    \mathcal L(\theta) = \min_{\theta} \|  \mathcal T_{\boldsymbol{\theta}}(g)(\boldsymbol{x}, t) -  \mathcal T(g)(\boldsymbol{x}, t) \|^2_2,
\end{aligned}
\end{equation}
where $\mathcal T_{\boldsymbol{\theta}}(g)(\boldsymbol{x}, t)$ is the prediction of the DeepONet and $\mathcal T(g)(\boldsymbol{x}, t)$ is the ground truth of the solution obtained from either experimental measures or by using a standard numerical solver. The branch and trunk networks can be modeled with any DNN architecture. Here, we consider a fully connected feed-forward neural network for both networks (branch and trunk).

\section{Worst-Case Performance}

\begin{table}[h]
    \caption{Worst-Case Mean squared error obtained for all applications presented in this work using Function Encoders and DeepONet. Max is computed over the last 10 evaluations over 10 seeds for each algorithm. }
    % \fontsize{6.5pt}{6.5pt}\selectfont
    \label{tab:worst_case_mse}
    \centering
    \begin{tabular}{lcccccc} %|l|c|c|c|c|c|c|c|c|} 
        \toprule
        \multicolumn{1}{c}{Dataset} & \multicolumn{3}{c}{Function Encoders} & \multicolumn{3}{c}{DeepONet} \\
        \cmidrule{1-7}
            & B2B & SVD & Eigen & Vanilla & POD & Two-stage\\
        \midrule
        Anti-Derivative & $\mathbf{5.84\mathrm{e}{-02}}$ & $3.51\mathrm{e}{+00}$ & $9.44\mathrm{e}{+00}$ & $9.61\mathrm{e}{-01}$ & $2.25\mathrm{e}{+03}$ & $3.98\mathrm{e}{-01}$ \\
        %\midrule
        Derivative  & $\mathbf{2.18\mathrm{e}{-03}}$ & $6.25\mathrm{e}{-02}$ & $1.30\mathrm{e}{-02}$ & $1.09\mathrm{e}{-02}$ & $1.07\mathrm{e}{+01}$ & $3.90\mathrm{e}{-03}$ \\
        %\midrule
        1D Darcy Flow & $\mathbf{2.66\mathrm{e}{-05}}$ & $1.03\mathrm{e}{-03}$ &  -  & $6.26\mathrm{e}{-05}$ & $4.72\mathrm{e}{-05}$ & $4.59\mathrm{e}{-04}$ \\
        %\midrule
        2D Darcy Flow  & $\mathbf{7.53\mathrm{e}{-03}}$ & $3.28\mathrm{e}{-02}$ &  -  & $3.13\mathrm{e}{-02}$ & $2.86\mathrm{e}{-02}$ & $1.66\mathrm{e}{-02}$ \\
        %\midrule
        Elastic Plate & $\mathbf{1.85\mathrm{e}{-04}}$ & $1.50\mathrm{e}{-01}$ &  -  & $2.84\mathrm{e}{-03}$ & $3.99\mathrm{e}{-03}$ &  -    \\
        %\midrule
        Parameterized Heat Equation  & $\mathbf{1.14\mathrm{e}{-03}}$ & $2.70\mathrm{e}{-01}$ &  -  & $3.78\mathrm{e}{-03}$ & $1.05\mathrm{e}{+00}$ &  -  \\
        Burger's Equation & $\mathbf{9.46\mathrm{e}{-04}}$ & $1.21\mathrm{e}{-01}$ &  -  & $3.18\mathrm{e}{-03}$ & $2.27\mathrm{e}{+00}$ & $2.40\mathrm{e}{+00}$ \\
        \bottomrule
    \end{tabular}

\end{table}

\end{document}